%% file: main.tex
\pdfoutput=1
\documentclass{article}

\usepackage[accepted]{icml2023}

\usepackage{math}
\usepackage{paper}

\newtheorem{theorem}{Theorem}
\newtheorem{proposition}[theorem]{Proposition}
\newtheorem{lemma}[theorem]{Lemma}

\theoremstyle{definition}
\newtheorem{definition}{Definition}
\newtheorem{assumption}{Assumption}
\newtheorem*{example}{Example}

\newtheorem{disentanglement}{Disentanglement}
\newtheorem{disentanglementsub}{Disentanglement}[disentanglement]
\newenvironment{disentanglementsubalt}[1]{%
\addtocounter{disentanglementsub}{-1}
\renewcommand{\thedisentanglementsub}{\ref*{#1}$'$}%
\begin{disentanglementsub}
}{%
\end{disentanglementsub}
}

\crefformat{disentanglement}{\textbf{D}.~#2#1#3}
\crefformat{disentanglementsub}{\textbf{D}.~#2#1#3}
\crefformat{disentanglementsubalt}{\textbf{D}.~#2#1#3}

\surroundwithmdframed[innerleftmargin=1ex, innerrightmargin=1ex, innertopmargin=1.5ex, innerbottommargin=1ex, linewidth=1.5pt]{disentanglement}
\surroundwithmdframed[innerleftmargin=1ex, innerrightmargin=1ex, innertopmargin=1.5ex, innerbottommargin=1ex, linewidth=1pt, linecolor=white, tikzsetting={draw=black, dashed}]{disentanglementsub}

\newcommand{\say}[1]{``\textit{#1}''}
\newcommand{\omitted}[1]{\dotuline{#1}}
\newcommand{\newterm}[1]{\textcolor{\green}{\textit{#1}}}
\newcommand{\catterm}[1]{\textcolor{\red}{\textit{#1}}}
\newcommand{\epigraph}[1]{{\centering\uline{\textit{#1}}\par}}

\begin{document}

\setlength{\abovedisplayskip}{\lineskip}
\setlength{\belowdisplayskip}{\lineskip}

%%%%%%%%%%%%%%%%%%%%%%%%%%%%%%%%%%%%%%%%%%%%%%%%%%
% meta
%%%%%%%%%%%%%%%%%%%%%%%%%%%%%%%%%%%%%%%%%%%%%%%%%%

\twocolumn[
\icmltitle{A Category-theoretical Meta-analysis of Definitions of Disentanglement}
\begin{icmlauthorlist}
\icmlauthor{Yivan Zhang}{utokyo,riken}
\icmlauthor{Masashi Sugiyama}{riken,utokyo}
\end{icmlauthorlist}
\icmlaffiliation{utokyo}{The University of Tokyo, Tokyo, Japan}
\icmlaffiliation{riken}{RIKEN AIP, Tokyo, Japan}
\icmlcorrespondingauthor{Yivan Zhang}{yivanzhang@ms.k.u-tokyo.ac.jp}
\icmlkeywords{Category Theory, Disentanglement}
\vskip 0.3in
]

%%%%%%%%%%%%%%%%%%%%%%%%%%%%%%%%%%%%%%%%%%%%%%%%%%
% main
%%%%%%%%%%%%%%%%%%%%%%%%%%%%%%%%%%%%%%%%%%%%%%%%%%

\printAffiliationsAndNotice{}

\begin{abstract}
\input{sections/abstract.tex}
\end{abstract}

\section{Introduction}
\input{sections/introduction.tex}

\section{Product: Core of Disentanglement}
\label{sec:product}
\input{sections/product.tex}

\section{Sets and Functions}
\label{sec:set}
\input{sections/set.tex}

\section{Algebra Actions and Equivariant Maps}
\label{sec:algebra}
\input{sections/algebra.tex}

\section{Sets and Relations}
\label{sec:relation}
\input{sections/relation.tex}

\section{Measurable Spaces and Stochastic Maps}
\label{sec:probability}
\input{sections/probability.tex}

\section{Limitations}
\input{sections/limitations.tex}

\section{Conclusion}
\input{sections/conclusion.tex}

\section*{Acknowledgements}
\input{sections/acknowledgements.tex}

\clearpage
\bibliography{references}
\bibliographystyle{icml2023}

\clearpage
\appendix
\input{sections/appendix.tex}

\end{document}

%% file: sections/abstract.tex
Disentangling the factors of variation in data is a fundamental concept in machine learning and has been studied in various ways by different researchers, leading to a multitude of definitions.
Despite the numerous empirical studies, more theoretical research is needed to fully understand the defining properties of disentanglement and how different definitions relate to each other.
This paper presents a meta-analysis of existing definitions of disentanglement, using category theory as a unifying and rigorous framework.
We propose that the concepts of the cartesian and monoidal products should serve as the core of disentanglement.
With these core concepts, we show the similarities and crucial differences in dealing with (i) functions, (ii) equivariant maps, (iii) relations, and (iv) stochastic maps.
Overall, our meta-analysis deepens our understanding of disentanglement and its various formulations and can help researchers navigate different definitions and choose the most appropriate one for their specific context.

%% file: sections/introduction.tex
%%%%%%%%%%%%%%%%%%%%%%%%%%%%%%%%%%%%%%%%%%%%%%%%%%
% What is disentanglement?
%%%%%%%%%%%%%%%%%%%%%%%%%%%%%%%%%%%%%%%%%%%%%%%%%%

\emph{Disentanglement}, in machine learning, refers to the ability to identify and separate the underlying factors that contribute to a particular variation in data \citep{bengio2013representation}.
It is a process of breaking down a complex phenomenon into simpler components.
It has been suggested that disentangled representation learning is a promising way toward reliable, interpretable, and data-efficient machine learning \citep{locatello2019challenging, montero2020role, dittadi2021transfer}.

%%%%%%%%%%%%%%%%%%%%%%%%%%%%%%%%%%%%%%%%%%%%%%%%%%
% Definitions
%%%%%%%%%%%%%%%%%%%%%%%%%%%%%%%%%%%%%%%%%%%%%%%%%%

Because disentanglement is an important concept, many researchers have approached this problem from different angles, resulting in various definitions, metrics, methods, and models.
Some definitions are based on the intuition that:
(1.~modularity) a change in one factor should lead to a change in a single code;
(2.~compactness/completeness) a factor should be associated with only one code; and
(3.~explicitness/informativeness) the code should be able to predict the factor
\citep{ridgeway2018learning, eastwood2018framework}.
Another line of research is based on group theory and representation theory \citep{cohen2014learning, cohen2015transformation, higgins2018towards}, where the mapping from the data to the code is required to be equivariant to product group actions, preserving the product structure of automorphisms (a.k.a.~symmetries).
Meanwhile, information theory \citep{chen2018isolating} and invariance \citep{higgins2017betavae} also play an important role in characterizing disentanglement.

%%%%%%%%%%%%%%%%%%%%%%%%%%%%%%%%%%%%%%%%%%%%%%%%%%
% Why do we need a meta-analysis?
%%%%%%%%%%%%%%%%%%%%%%%%%%%%%%%%%%%%%%%%%%%%%%%%%%

Then why do we want to conduct a \emph{meta-analysis}?
Because we study the theories and techniques of disentanglement, yet our definitions of it are quite \emph{entangled}.
Although large-scale experimental studies exist \citep{locatello2019challenging}, theoretical analyses and systematic comparisons are limited \citep{sepliarskaia2019evaluating, carbonneau2022measuring}.
Several important questions remain to be answered:
\begin{itemize}
\item What are the defining properties of disentanglement?
\item What operations and structures are essential, and what are specific to the task?
\item Given two definitions or metrics, does one imply the other in any situation?
\item Are the existing algebraic and statistical approaches compatible with one another?
\end{itemize}
Things quickly become complicated without an abstract language to describe existing results.

%%%%%%%%%%%%%%%%%%%%%%%%%%%%%%%%%%%%%%%%%%%%%%%%%%
% Why do we use category theory?
%%%%%%%%%%%%%%%%%%%%%%%%%%%%%%%%%%%%%%%%%%%%%%%%%%

\emph{Category theory} \citep{borceux1994handbook, awodey2006category, leinster2014basic} is particularly suitable for designing and organizing a system of this level of complexity.
It has found applications in many scientific fields \citep{baez2017applied, bradley2018applied, fong2019invitation}, recently also in machine learning \citep{gavranovic2019compositional, de2020natural, shiebler2021category, dudzik2022graph}.
In this work, we aim to \emph{disentangle the definitions of disentanglement} from a categorical perspective.

%%%%%%%%%%%%%%%%%%%%%%%%%%%%%%%%%%%%%%%%%%%%%%%%%%
% Contributions
%%%%%%%%%%%%%%%%%%%%%%%%%%%%%%%%%%%%%%%%%%%%%%%%%%

In \cref{sec:product}, we first introduce the essential concepts of the \emph{cartesian product} and \emph{monoidal product}, which we argue should be the core of disentanglement.
Next, we look into the requirements based on examples and counterexamples through \cref{sec:set,sec:algebra,sec:relation,sec:probability}.
We use the categories of
(1.~$\cSet$) sets and functions to define the concepts of modularity and explicitness as the defining properties of disentanglement \citep{ridgeway2018learning};
(2.~$[\cS, \cSet]$) functors and natural transformations to generalize to actions of an algebra (monoid, group, etc.) and equivariant maps \citep{higgins2018towards};
(3.~$\cRel$) sets and relations as an example of a symmetric monoidal category; and
(4.~$\cStoch$) measurable spaces and stochastic maps to introduce the concept of the Markov category \citep{fritz2020synthetic} and explain how we should use the copy/delete/projection operations to characterize disentanglement.
A full-blown example is given in the end.

It is worth clarifying that this paper does \emph{not} discuss metrics, models, methods, supervision, and learnability.
Also, our contribution is \emph{not} to category theory itself, as the math we used is not new.
However, our work shows how category theory can transfer and integrate knowledge across disciplines and how abstract definitions can simplify a complex system \citep{baez2017applied}.
We hope our work is an initial step toward a full understanding of disentanglement.

%% file: sections/product.tex
In this section, we briefly review two important categorical concepts --- the \emph{cartesian product} and \emph{monoidal product}, which are the core of the disentanglement.
We will omit many basic concepts such as the \omitted{category, functor, natural transformation, and monad}.
Note that we frequently use \catterm{commutative diagrams} \citep{awodey2006category} and \catterm{string diagrams} \citep{selinger2010survey} as graphical calculus (See \cref{app:diagram}).

%%%%%%%%%%%%%%%%%%%%%%%%%%%%%%%%%%%%%%%%%%%%%%%%%%

\subsection{Cartesian Category}

Let us dive into the definition of the (cartesian) product:
\begin{definition}[Product]
In any \catterm{category} $\cC$, a \catterm{product} of two \catterm{objects} $A$ and $B$ is an object $A \times B$, together with two \catterm{morphisms} $A \xot{p_1} A \times B \xto{p_2} B$, called \catterm{projections}, \omitted{satisfying the \catterm{universal property}}:
\begin{equation}
\label{eq:product}
\begin{tikzcd}[column sep=2em, row sep=2em]
&
C
\arrow[d, "\angles{f_1, f_2}" {description, inner sep=1pt}, unique morphism]
\arrow[ld, "f_1"']
\arrow[rd, "f_2"]
\\
A
&
A \times B
\arrow[l, "p_1"]
\arrow[r, "p_2"']
&
B
\end{tikzcd}
\end{equation}
Given any object $C$ and morphisms $A \xot{f_1} C \xto{f_2} B$, there exists a \emph{unique} morphism $\angles{f_1, f_2}: C \to A \times B$, called a \catterm{paring} of $f_1$ and $f_2$, such that $f_1 = p_1 \compL \angles{f_1, f_2}$ and $f_2 = p_2 \compL \angles{f_1, f_2}$.
\end{definition}

The gist is that any morphism $C \xto{f} A \times B$ \emph{to} a product is merely a pair of component morphisms $A \xot{f_1} C \xto{f_2} B$, and all such morphisms arise this way.
However, note that a morphism $A \times B \to C$ \emph{from} a product can depend on both components.

We will be needing the following definitions and properties:
\begin{itemize}
\item The \catterm{product morphism} of $f: A \to C$ and $g: B \to D$ is defined as $f \times g: A \times B \to C \times D \defeq \angles{f \compL p_1, g \compL p_2}$, which makes product $\times: \cC \times \cC \to \cC$ a \catterm{bifunctor}.
\item The \catterm{diagonal morphism} of an object $A$ is defined as $\Delta_A: A \to A \times A \defeq \angles{\id_A, \id_A}$, which \say{duplicates} $A$.
\item The \catterm{terminal object} $1$, if exists, is the \catterm{unit} of the product: for any object $A$, there is a \emph{unique} \catterm{terminal morphism} $e_A: A \to 1$, which \say{deletes} $A$, and $A \times 1 \iso A \iso 1 \times A$.
\item The product is \catterm{associative} \omitted{up to isomorphism} $\alpha_{A, B, C}: (A \times B) \times C \iso A \times (B \times C) \defeq \angles{p_1 \compL p_1, p_2 \times \id_C}$, which allows us to define products $\prod_{i=1}^N A_i = A_1 \times \dots \times A_N$ and projections $p_i: \prod_{i=1}^N A_i \to A_i$ for $N \geq 2$ objects.
We use subscript $f_i \defeq p_i \compL f$ as an abbreviation.
\item The product is \catterm{commutative} \omitted{up to isomorphism} $\beta_{A, B}: A \times B \iso B \times A \defeq \angles{p_2, p_1}$.
\end{itemize}

A \catterm{cartesian category} is a category with all finite products, i.e., all binary products and a terminal object.

%%%%%%%%%%%%%%%%%%%%%%%%%%%%%%%%%%%%%%%%%%%%%%%%%%

\subsection{Monoidal Category}

Having all products is sometimes too strong a condition.
Besides, the product, if exists, is not always an appropriate concept for disentanglement.
Therefore, sometimes we need to consider a weaker notion of the ``product'':
\begin{definition}[Symmetric monoidal category]
A \catterm{symmetric monoidal category} $(\cC, \otimes, I)$ is a category $\cC$ equipped with a \catterm{monoidal product} $\otimes: \cC \times \cC \to \cC$ and a \catterm{monoidal unit} $I$, which is unital, associative, and commutative \omitted{up to natural isomorphisms} and \omitted{subject to some coherence conditions}.
\end{definition}

The monoidal products are weaker because they do not need to satisfy the universal property, so there are no canonical projections anymore.
A \catterm{cartesian (monoidal) category} is a symmetric monoidal category whose monoidal product is given by the cartesian product.
However, many interesting monoidal categories are not cartesian.

Some symmetric monoidal categories have extra structures or properties, including
\begin{itemize}
\item \catterm{monoidal category with diagonals} $\Delta_A: A \to A \otimes A$, which is natural in $A$ if
\begin{equation}
\label{eq:diagonal}
\begin{tikzcd}[column sep=2.2em, row sep=1em]
A
\arrow[r, "f"]
\arrow[d, "\Delta_A"']
&
B
\arrow[d, "\Delta_B"]
\\
A \otimes A
\arrow[r, "f \otimes f"]
&
B \otimes B
\end{tikzcd}
\hspace{2em}
\begin{tikzpicture}
  \begin{pgfonlayer}{nodelayer}
    \node at (0, 0) {$=$};
    \node (1)  at (-1  , -.6) {};
    \node (2)  at (-1  , -.3) [style=diagonal] {};
    \node (3)  at (-1.4,   0) {};
    \node (4)  at (-1.4,  .6) {};
    \node (5)  at (- .6,   0) {};
    \node (6)  at (- .6,  .6) {};
    \node (7)  at (  .8, -.6) {};
    \node (8)  at (  .8,  .3) [style=diagonal] {};
    \node (9)  at (  .4,  .6) {};
    \node (10) at ( 1.2,  .6) {};
    \node at (-1.4,  .2) [style=morphism] {$f$};
    \node at (- .6,  .2) [style=morphism] {$f$};
    \node at (  .8, -.2) [style=morphism] {$f$};
  \end{pgfonlayer}
  \begin{pgfonlayer}{edgelayer}
    \draw (1.center) to (2.center);
    \draw [out=180, in=-90] (2.center) to (3.center);
    \draw [out=0, in=-90] (2.center) to (5.center);
    \draw (3.center) to (4.center);
    \draw (5.center) to (6.center);
    \draw (7.center) to (8.center);
    \draw [out=180, in=-90] (8.center) to (9.center);
    \draw [out=0, in=-90] (8.center) to (10.center);
  \end{pgfonlayer}
\end{tikzpicture}
\end{equation}

\item \catterm{semicartesian (monoidal) category}, whose monoidal unit $I$ is a terminal object:
\begin{equation}
\label{eq:terminal}
\begin{tikzcd}[column sep=0em, row sep=1em]
A
\arrow[rr, "f"]
\arrow[rd, "e_A"', unique morphism, pos=.1]
&&
B
\arrow[ld, "e_B", unique morphism, pos=.1]
\\
&
I
\end{tikzcd}
\hspace{5em}
\begin{tikzpicture}
  \begin{pgfonlayer}{nodelayer}
    \node at (0, 0) {$=$};
    \node (1)  at (-.8, -.5) {};
    \node (2)  at (-.8,  .5) [style=diagonal] {};
    \node (3)  at ( .8, -.5) {};
    \node (4)  at ( .8,  .5) [style=diagonal] {};
    \node at (-.8, 0) [style=morphism] {$f$};
  \end{pgfonlayer}
  \begin{pgfonlayer}{edgelayer}
    \draw (1.center) to (2.center);
    \draw (3.center) to (4.center);
  \end{pgfonlayer}
\end{tikzpicture}
\end{equation}

\item \catterm{monoidal category with projections} $\pi_1: A \otimes B \to A$ and $\pi_2: A \otimes B \to B$ \citep{franz2002stochastic, leinster2016monoidal}, and

\item \catterm{Markov category} \citep[Definition 2.1]{fritz2020synthetic}.
\end{itemize}

They have the following relationship:
\begin{equation}
\begin{tikzcd}[column sep=1em, row sep=1em]
\text{cartesian}
\arrow[r, phantom, "\subset"]
\arrow[d, phantom, "\subset" {rotate=-90}]
&
\text{Markov}
\arrow[r, phantom, "\subset"]
&
\text{semicartesian}
\arrow[d, phantom, "=" {rotate=90}]
\\
\text{diagonals}
\arrow[r, phantom, "\subset"]
&
\text{monoidal}
\arrow[r, phantom, "\supset"]
&
\text{projections}
\end{tikzcd}
\end{equation}
These structures and properties will be important in the rest of this paper.

%% file: sections/set.tex
Equipped with these concepts, let us now look at the definitions of disentanglement.
$\cSet$, the category of sets and functions, serves as our primary example.
$\cSet$ is cartesian, whose product is given by the \emph{Cartesian product} of sets.

We use $[1..N]$ to denote the set of numbers from $1$ to $N$.
We use $\setminus i$ as an abbreviation of $[1..N] \setminus \set{i}$, i.e., the set of numbers from $1$ to $N$ except $i$.

%%%%%%%%%%%%%%%%%%%%%%%%%%%%%%%%%%%%%%%%%%%%%%%%%%

\subsection{Generating Process}
\label{ssec:gen}

First, let us consider how the data is generated from a set of factors.
If all combinations of factors are equally possible (cf.~\cref{sec:relation}), we can assume that

\begin{assumption}
\label{asm:yproduct}
The set of \newterm{factors} $Y \defeq \prod_{i=1}^N Y_i$ is a \emph{product} of $N$ sets.
\end{assumption}

%%%%%%%%%%%%%%%%%%%%

Then, let $X$ be the set of \newterm{observations}.
A \newterm{generating process} $g: Y \to X$ is simply \emph{a morphism from a product}, i.e., a function with multiple inputs.
It is an \say{entangling process} because we do not have any structural assumptions on $X$.
However, we need some basic requirements for $g$ to ensure that the analysis is meaningful.
For starters, we assume that

\begin{assumption}
\label{asm:mono}
$g: Y \mono X$ is a \catterm{monomorphism}.
\end{assumption}

This means that if two observations are the same, their underlying factors must be the same, too.
This assumption avoids the model not satisfying a disentanglement definition simply because of a wrong choice of factors.
% For example, natural language texts may not be a good candidate as factors for evaluating disentanglement because there could be multiple ways to describe the same thing.

%%%%%%%%%%%%%%%%%%%%%%%%%%%%%%%%%%%%%%%%%%%%%%%%%%

\subsection{Encoding Process}
\label{ssec:enc}

Next, we consider how an \newterm{encoding process} $f: X \to Z$ can exhibit disentanglement and what desiderata are.
Following \citet{ridgeway2018learning} and \citet{eastwood2018framework}, we call $Z$ the set of \newterm{codes}, which should also be a product.
In this work, we consider a simple case where

\begin{assumption}
\label{asm:zproduct}
The codes $Z$ also have $N$ components, and the code projections $p_i: Z \to Z_i$ are known a priori.
\end{assumption}

%%%%%%%%%%%%%%%%%%%%

Based on \cref{asm:zproduct}, we present our first definition:

\begin{disentanglement}[A morphism to a product]
\label{dis:product}
In a category $\cC$, a disentangled encoding process is a morphism $f: X \to Z$ to a \emph{product} $Z \defeq \prod_{i=1}^N Z_i$.
\end{disentanglement}

This is perhaps the minimal requirement for an encoder to exhibit some level of disentanglement.
It means that the encoder outputs multiple components, and we can extract each component without losing any information.
Note that \cref{dis:product} does not even rely on the ground-truth factors $Y$ and a generating process $g$.\footnote{\cref{dis:product} refers to \textbf{Disentanglement}~\ref{dis:product}.}

%%%%%%%%%%%%%%%%%%%%

Let us now improve \cref{dis:product}.
A disentanglement requirement that many researchers agree on is \newterm{modularity}, such that \say{each code conveys information about at most one factor} \citep{ridgeway2018learning}.
It is natural to consider the composition $m: Y \to Z \defeq f \compL g$ of a generating process $g$ and an encoding process $f$, which we call a \newterm{code generating process} (w.r.t.~a given encoding $f$), while $g: Y \mono X$ can be referred to as a \newterm{data generating process}.
Then, modularity is a property of a code generating process:

\begin{disentanglementsub}
\label{dis:mproduct}
$m = \prod_{i=1}^N (m_{i,i}: Y_i \to Z_i)$.
\begin{equation*}
\begin{tikzpicture}
  \begin{pgfonlayer}{nodelayer}
    \node at (0, 0) {$=$};
    \node (1)  at (-3, -.4) {};
    \node (2)  at (-2, -.4) {};
    \node (3)  at (-1, -.4) {};
    \node (4)  at (-3,  .4) {};
    \node (5)  at (-2,  .4) {};
    \node (6)  at (-1,  .4) {};
    \node (7)  at ( 1, -.4) {};
    \node (8)  at ( 2, -.4) {};
    \node (9)  at ( 3, -.4) {};
    \node (10) at ( 1,  .4) {};
    \node (11) at ( 2,  .4) {};
    \node (12) at ( 3,  .4) {};
    \node at (-3, -.6) {$Y_1$};
    \node at (-2, -.6) {$Y_2$};
    \node at (-1, -.6) {$Y_3$};
    \node at (-3,  .6) {$Z_1$};
    \node at (-2,  .6) {$Z_2$};
    \node at (-1,  .6) {$Z_3$};
    \node at ( 1, -.6) {$Y_1$};
    \node at ( 2, -.6) {$Y_2$};
    \node at ( 3, -.6) {$Y_3$};
    \node at ( 1,  .6) {$Z_1$};
    \node at ( 2,  .6) {$Z_2$};
    \node at ( 3,  .6) {$Z_3$};
    \node at (-2, 0) [style=morphism, minimum width=2.8cm] {$m: Y \to Z$};
    \node at ( 1, 0) [style=morphism, minimum width=.8cm] {$m_{1,1}$};
    \node at ( 2, 0) [style=morphism, minimum width=.8cm] {$m_{2,2}$};
    \node at ( 3, 0) [style=morphism, minimum width=.8cm] {$m_{3,3}$};
  \end{pgfonlayer}
  \begin{pgfonlayer}{edgelayer}
    \draw (1.center) to (4.center);
    \draw (2.center) to (5.center);
    \draw (3.center) to (6.center);
    \draw (7.center) to (10.center);
    \draw (8.center) to (11.center);
    \draw (9.center) to (12.center);
  \end{pgfonlayer}
\end{tikzpicture}
\end{equation*}
\say{The $i$-th code only encodes the $i$-th factor.}
\end{disentanglementsub}

Morphisms $m$, $m_i$, and $m_{i,i}$ have the following relationship:

\begin{proposition}
\label{prop:m}
$\forall i \in [1..N].\; m_i \defeq p_i \compL m = m_{i,i} \compL p_i$.
\begin{equation}
\begin{tikzcd}[column sep=3em, row sep=2em]
Y
\arrow[r, "m"]
\arrow[d, "p_i"']
\arrow[rd, "m_i" description]
&
Z
\arrow[d, "p_i"]
\\
Y_i
\arrow[r, "m_{i,i}"']
&
Z_i
\end{tikzcd}
\end{equation}
\end{proposition}

%%%%%%%%%%%%%%%%%%%%

\cref{dis:mproduct} is straightforward and intuitive, but there is one difficulty: it relies on the \emph{existence} of some other morphisms.
Given $m$, verifying if $m_{i,i}$ exists is not trivial.
Although, if \cref{dis:mproduct} holds, we can construct $m_{i,i}$ from $m$ as follows:

\begin{proposition}
\label{prop:mii}
$\forall i \in [1..N].\; \forall y_i: 1 \to Y_i.\;
m_{i,i} = Y_i \xto{\iso} 1 \times \dots \times Y_i \times \dots \times 1 \xto{y_1 \times \dots \times \id_{Y_i} \times \dots \times y_N} Y \xto{m} Z \xto{p_i} Z_i$.
\end{proposition}

In words, we can choose other factors arbitrarily, and a modular encoder should give us the same code.
This inspires us to have a more \emph{verifiable} definition as follows.

%%%%%%%%%%%%%%%%%%%%

A good property of $\cSet$ is that it is \catterm{cartesian closed}, i.e., it has \catterm{exponential objects}, given by the sets of functions.
Let $\widehat{m_i}: Y_{\setminus i} \to {Z_i}^{Y_i}$ be the \catterm{exponential transpose} (\emph{currying}) of $m_i: Y \to Z_i$.
To check modularity, we can verify if

\begin{disentanglementsub}
\label{dis:const}
$\widehat{m_i}$ is a \catterm{constant morphism}.
\end{disentanglementsub}

Therefore, we can obtain the exponential transpose first and check whether it is constant.
Even better, we can guarantee that these definitions are equivalent:

\begin{theorem}
\label{thm:exp}
$\cref{dis:mproduct} \leqv \cref{dis:const}$.
\end{theorem}

\begin{proof}
Diagram chase.
\begin{equation}
\begin{tikzcd}[column sep=2.5em, row sep=1.5em]
&&
1
\arrow[rd, "\widehat{m_{i,i}}"]
\\
Y_{\setminus i}
\arrow[rrr, "\widehat{m_i}"', unique morphism]
\arrow[rru, unique morphism]
&&&
{Z_i}^{Y_i}
\\[-1.5em]
&&
Y_i \times 1
\arrow[uu, "p_2"']
\arrow[rdd, "\id_{Y_i} \times \widehat{m_{i,i}}", pos=.1]
\\[-1.5em]
&
Y_i
\arrow[ru, "\iso"]
\arrow[dd, "{m_{i,i}}"']
\\[-1.5em]
Y_i \times Y_{\setminus i}
\arrow[rd, "m_i"']
\arrow[rrr, "\id_{Y_i} \times \widehat{m_i}"]
\arrow[uuu, "p_2"]
\arrow[ru, "p_1"]
&&&
Y_i \times {Z_i}^{Y_i}
\arrow[lld, "\epsilon_{Y_i}"]
\arrow[uuu, "p_2"']
\\
&
Z_i
\end{tikzcd}
\end{equation}
\end{proof}

%%%%%%%%%%%%%%%%%%%%

Up to this point, we defined modularity in a cartesian closed category like $\cSet$.
However, we point out that modularity alone is not sufficient:

\begin{example}[Constant]
Let $Z$ be the terminal object $1^N \iso 1$.
The terminal morphism $e_Y: Y \to 1$ satisfies \cref{dis:mproduct}.
\end{example}

That is, \emph{an encoder sending everything to singletons is perfectly modular but also completely useless}.
Therefore, in addition to modularity, we should measure how useful and informative the codes are.

%%%%%%%%%%%%%%%%%%%%%%%%%%%%%%%%%%%%%%%%%%%%%%%%%%

\subsection{Decoding Process}
\label{ssec:dec}

This is where the concepts of \newterm{explicitness} \citep{ridgeway2018learning} or \newterm{informativeness} \citep{eastwood2018framework} come in, meaning that \say{the factors can be precisely determined from the codes}.
It might be tempting to define explicitness as

\begin{disentanglementsub}
\label{dis:inverse}
$f$ is an \catterm{inverse} of $g$.
\end{disentanglementsub}

Then, the factors can be completely reconstructed from the observations.
A drawback of \cref{dis:inverse} is that it requires the code set $Z$ to be the same as the factor set $Y$, so $Y$ needs to be known during training.
However, it is common that an encoder $f: X \to Z$ is trained with self-supervision or weak supervision \citep{shu2020weakly, wang2021self}, and the ground-truth factors $Y$ are only available during evaluation.

%%%%%%%%%%%%%%%%%%%%

Therefore, we weaken the requirement and define the explicitness of a code generating process as

\begin{disentanglementsub}
\label{dis:mmono}
$m$ is a \catterm{split monomorphism}.
\begin{equation*}
\begin{tikzpicture}
  \begin{pgfonlayer}{nodelayer}
    \node at (0, 0) {$=$};
    \node (1)  at (-3, -.8) {};
    \node (2)  at (-2, -.8) {};
    \node (3)  at (-1, -.8) {};
    \node (4)  at (-3,  .7) {};
    \node (5)  at (-2,  .7) {};
    \node (6)  at (-1,  .7) {};
    \node (7)  at ( 1, -.8) {};
    \node (8)  at ( 2, -.8) {};
    \node (9)  at ( 3, -.8) {};
    \node (10) at ( 1,  .7) {};
    \node (11) at ( 2,  .7) {};
    \node (12) at ( 3,  .7) {};
    \node at (-2, -.4) [style=morphism, minimum width=2.8cm] {$m$};
    \node at (-2,  .3) [style=morphism, minimum width=2.8cm] {$h$};
  \end{pgfonlayer}
  \begin{pgfonlayer}{edgelayer}
    \draw (1.center) to (4.center);
    \draw (2.center) to (5.center);
    \draw (3.center) to (6.center);
    \draw (7.center) to (10.center);
    \draw (8.center) to (11.center);
    \draw (9.center) to (12.center);
  \end{pgfonlayer}
\end{tikzpicture}
\end{equation*}
\say{The codes encode the factors faithfully.}
\end{disentanglementsub}

This means that there exists a morphism $h: Z \to Y$, which we call a \newterm{decoding process}, such that $h \compL m = \id_Y$.
In other words, $h$ is a \catterm{retraction} of $m$.
To summarize, we will focus on the following morphisms from now:

\begin{equation}
\begin{tikzcd}[column sep=5.4em]
Y
\arrow[r, "g \text{ (generating)}" description]
\arrow[rr, "m" description, bend left=10, looseness=.5, start anchor={[yshift=2mm]}, end anchor={[yshift=2mm]}]
\arrow[rrr, "\id_Y" description, bend right=10, looseness=.5, start anchor={[yshift=-2mm]}, end anchor={[yshift=-2mm]}]
&
X
\arrow[r, "f \text{ (encoding)}" description]
&
Z
\arrow[r, "h \text{ (decoding)}" description]
&
Y
\end{tikzcd}
\end{equation}

%%%%%%%%%%%%%%%%%%%%

Note that explicitness only indicates if the factors can be recovered.
We may end up with entangled codes:

\begin{example}[Rotation]
Let $Y$ be a vector space.
A rotation is an invertible linear transformation and satisfies \cref{dis:mmono}.
\end{example}

%%%%%%%%%%%%%%%%%%%%

To avoid this, we may want the decoder to be modular, too.
This property is related to the concepts of \newterm{compactness} \citep{ridgeway2018learning} and \newterm{completeness} \citep{eastwood2018framework}, meaning that \say{a factor is associated with only one code} (See also \cref{app:compact}).
Like \cref{dis:mproduct}, we can require $h$ to be a product morphism:

\begin{disentanglementsub}
\label{dis:hproduct}
$h = \prod_{i=1}^N (h_{i,i}: Z_i \to Y_i)$.
\begin{equation*}
\begin{tikzpicture}
  \begin{pgfonlayer}{nodelayer}
    \node at (0, 0) {$=$};
    \node (1)  at (-3, -.8) {};
    \node (2)  at (-2, -.8) {};
    \node (3)  at (-1, -.8) {};
    \node (4)  at (-3,  .7) {};
    \node (5)  at (-2,  .7) {};
    \node (6)  at (-1,  .7) {};
    \node (7)  at ( 1, -.8) {};
    \node (8)  at ( 2, -.8) {};
    \node (9)  at ( 3, -.8) {};
    \node (10) at ( 1,  .7) {};
    \node (11) at ( 2,  .7) {};
    \node (12) at ( 3,  .7) {};
    \node at (-2, -.4) [style=morphism, minimum width=2.8cm] {$m$};
    \node at (-2,  .3) [style=morphism, minimum width=2.8cm] {$h$};
    \node at ( 2, -.4) [style=morphism, minimum width=2.8cm] {$m$};
    \node at (1, .3) [style=morphism, minimum width=.8cm] {$h_{1,1}$};
    \node at (2, .3) [style=morphism, minimum width=.8cm] {$h_{2,2}$};
    \node at (3, .3) [style=morphism, minimum width=.8cm] {$h_{3,3}$};
  \end{pgfonlayer}
  \begin{pgfonlayer}{edgelayer}
    \draw (1.center) to (4.center);
    \draw (2.center) to (5.center);
    \draw (3.center) to (6.center);
    \draw (7.center) to (10.center);
    \draw (8.center) to (11.center);
    \draw (9.center) to (12.center);
  \end{pgfonlayer}
\end{tikzpicture}
\end{equation*}
\say{The $i$-th code encodes the $i$-th factor faithfully.}
\end{disentanglementsub}

If an encoder has a modular decoder, we can safely drop other codes if a downstream task only relies on a subset of factors.
For example, if a task only depends on factor $Y_i$, then a component encoder $f_i: X \to Z_i$ can encode sufficient information for this task.

%%%%%%%%%%%%%%%%%%%%

We point out that an encoder with a modular decoder may not be modular itself:

\begin{example}[Duplicate]
Let $Z$ be $Y \times Y$.
The diagonal morphism $\Delta_Y$ satisfies \cref{dis:hproduct} with a retraction $p_1 \times p_2$.
\end{example}

This means that a \emph{non-modular} and explicit encoder may copy all the factors for each code $Z_i \defeq Y$, and its modular decoder $h_{i,i}: Z_i \to Y_i \defeq p_i$ can simply project the code to each component, which is not what we expect.

A potential remedy to this issue is to require that the code does \emph{not} contain any other information except for the target factor:

\begin{disentanglementsub}
\label{dis:non-existence}
$\forall i, j \in [1..N].\; \nexists h_{i,j}: Z_i \to Y_j.$
$(i \neq j) \lcon (h_{i,j} \compL m_i = p_j).$

\say{The i-th code does not encode the j-th factor.}
\end{disentanglementsub}
However, \cref{dis:non-existence} is even harder to verify than \cref{dis:mproduct} because it relies on the \emph{non-existence} of some morphisms.
This is another difficulty in dealing with non-modular encoders.

Fortunately, we can guarantee that a \emph{modular} and explicit encoder must have a modular decoder:

\begin{theorem}
\label{thm:moddec}
$(\cref{dis:mproduct} \lcon \cref{dis:mmono}) \limp \cref{dis:hproduct}$.
\end{theorem}

%%%%%%%%%%%%%%%%%%%%

It is now clear that modularity (\cref{dis:mproduct}) and explicitness (\cref{dis:mmono}) of an encoder should be the defining properties of disentanglement and our main focus when designing and evaluating disentangled representation learning algorithms.
Waiving either of these requirements could cause problems.
Our analysis supports similar arguments made by \citet{ridgeway2018learning}, \citet{duan2019unsupervised}, and \citet{carbonneau2022measuring}.

%%%%%%%%%%%%%%%%%%%%

A minor issue is that a modular and explicit encoder may have a \say{non-explicit} decoder:

\begin{example}[Redundancy]
Let $Z$ be $(Y_1 \times Y_1) \times Y_2$.
The morphism $m = \Delta_{Y_1} \times \mathrm{id}_{Y_2}$ satisfies both \cref{dis:mproduct} and \cref{dis:mmono}.
\end{example}

It means that $Z_1 := Y_1 \times Y_1$ contains redundant information of $Y_1$.
All meaningful codes are of the form $((y_1, y_1), y_2)$, while codes of the form $((y_1, y_1'), y_2)$ are meaningless and should not be decoded.
In categorical terms, $m$ is a product morphism, a split monomorphism, but not an \catterm{epimorphism}.
If we want to \emph{traverse the code space}, we can additionally require $m$ to be a (split) epimorphism.

%% file: sections/algebra.tex
\epigraph{We can simply change the category from $\cSet$ to $[\cS, \cSet]$.}

In this section, we explain the above sentence by showing three ways to extend \cref{dis:product} and how it relates to the definition based on the direct product of groups \cite{higgins2018towards}.

%%%%%%%%%%%%%%%%%%%%

$[\cS, \cC]$ denotes the \catterm{functor category} of \catterm{functors} from $\cS$ to $\cC$ and \catterm{natural transformations} between these functors.
We call the category $\cS$ a \newterm{scheme}.
To see how it relates to the existing algebraic formulation of disentanglement, we need the following well-known fact:

\begin{definition}[Equivariance as naturality]
Many algebraic structures, such as monoids and groups, can be considered as \catterm{single-object categories}.
Then, an \emph{action} of an algebra at an object $A$ is precisely a \emph{functor} $F_A: \cS \to \cC$ from the corresponding scheme $\cS$ to a category $\cC$ containing $A$, and an \emph{equivariant map} $f: A \to B$ between two actions $F_A$ and $F_B$ is precisely a \emph{natural transformation} $\phi: F_A \nat F_B$.
\end{definition}

%%%%%%%%%%%%%%%%%%%%

An example is shown below:
\begin{equation}
\begin{tikzcd}[column sep=1em, row sep=2.5em]
\cS
\arrow[d, "F_A"'{name=A}, bend right=4em, pos=.55]
\arrow[d, "F_B"{name=B}, bend left=4em, pos=.55]
\arrow[from=A, to=B, "\phi", natural transformation, shorten=1mm]
&[5em]&
\pone
\arrow[loop, "a"', distance=1em, out=125, in=145]
\arrow[loop, "a \cdot b"', distance=1em, out=80, in=100]
\arrow[loop, "b"', distance=1em, out=35, in=55]
\arrow[ld, "F_A"', functor]
\arrow[rd, "F_B", functor]
\\
\cC
&
A
\arrow[loop, "a_A"', distance=1em, out=125, in=145]
\arrow[loop, "\substack{(a \cdot b)_A\\=a_A \compL b_A}"', distance=1em, out=170, in=190]
\arrow[loop, "b_A"', distance=1em, out=215, in=235]
\arrow[rr, "f \defeq \phi_\pone"']
&&
B
\arrow[loop, "a_B"', distance=1em, out=35, in=55]
\arrow[loop, "\substack{(a \cdot b)_B\\=a_B \compL b_B}"', distance=1em, out=-10, in=10]
\arrow[loop, "b_B"', distance=1em, out=305, in=325]
\end{tikzcd}
\end{equation}
We use subscript $a_A \defeq F_A a$ as an abbreviation.
We can see that $F_A$ and $F_B$ send the single $\cS$-object $\pone$ to $\cC$-objects $A$ and $B$ and send endomorphisms to endomorphisms.
In this way, we can consider $\cS$ as \catterm{syntax} and $\cC$ as \catterm{semantics}.

%%%%%%%%%%%%%%%%%%%%

\begin{example}[Regression vs.~Ranking]
Not all problems can be formulated using only endomorphisms, let alone groups.
Some \emph{ranking} problems \citep{liu2011learning} roughly correspond to finding order-preserving functions, which is equivariant to actions of the free monoid of natural numbers $\N$.
However, the usual \emph{regression} problems also require the preservation $f(x_0) = 0$ of the zero point, which is a nullary operation $\zero: 1 \to \N$ (a morphism from a singleton to the set $\N$).
\end{example}

%%%%%%%%%%%%%%%%%%%%%%%%%%%%%%%%%%%%%%%%%%%%%%%%%%

\subsection{Product Category and Functor Product}

Let us now consider the products of categories and functors.
We highlight the following two important properties:
\begin{itemize}
\item The category of small categories $\cCat$ is cartesian closed, with the product and exponential object given by the \catterm{product category} $\cS_1 \times \cS_2$ and functor category $[\cS, \cC]$.
\item If $\cC$ has \catterm{(co)limits} of a certain shape (e.g., product), then $[\cS, \cC]$ has pointwise (co)limits of the same shape (e.g., \catterm{functor product} $F_1 \times^\cS F_2: \cS \xto{\angles{F_1, F_2}} \cC \times \cC \xto{\times} \cC$).\footnote{We reserve the term \catterm{product functor} to the product morphism in $\cCat$, i.e., $F_1 \times F_2: \cS_1 \times \cS_2 \to \cC_1 \times \cC_2$, to avoid confusion.}
\end{itemize}

%%%%%%%%%%%%%%%%%%%%

% As a side note, it is important to distinguish between
% \begin{enumerate*}[(1)]
% \item product $\times$ in a category $\cC$ as a \emph{bifunctor},
% \item pointwise functor product $F_1 \times^\cS F_2$ as a $[\cS, \cC]$-\emph{object}, and
% \item product functor $G_1 \times G_2$ as a $\cCat$-\emph{morphism}.
% \end{enumerate*}

%%%%%%%%%%%%%%%%%%%%

Knowing if $\cC$ has products then so does $[\cS, \cC]$, we can now extend \cref{dis:product} straightforwardly by \emph{simply changing the category from $\cC$ to $[\cS, \cC]$}:

\begin{disentanglement}[A natural transformation to a functor product]
\label{dis:functor}
Let $\cS$ be a category, $\cC$ be a category with products, and $F_X, F_{Z_i}: \cS \to \cC, i \in [1..N]$ be functors.
A disentangled encoding process is a morphism to a product in $[\cS, \cC]$, i.e., a natural transformation $\phi: F_X \nat F_Z$ to a \emph{functor product} $F_Z \defeq \prod_{i=1}^N F_{Z_i}$.
\end{disentanglement}

In other words, the same scheme $\cS$ has $N$ different models via $F_{Z_i}$ in $\cC$, which are combined into a single model via product $F_Z$.
In the product group action example \citep{higgins2018towards}, \cref{dis:functor} means that the product group is viewed as a single-object category $\cS$, and the product structure of automorphisms is preserved via the functor product.

%%%%%%%%%%%%%%%%%%%%

Another approach is to view each group as a single-object category and the product group as a product category.
Then, we can use the following definition:

\begin{disentanglement}[A natural transformation between multifunctors]
\label{dis:multifunctor}
Let $\cS = \prod_{i=1}^N \cS_i$ be a product category, $\cC$ be a category, and $F_X, F_Z: \cS \to \cC$ be multifunctors.
A disentangled encoding process is a morphism in $[\cS, \cC]$, i.e., a natural transformation $\phi: F_X \nat F_Z$ between \emph{multifunctors}.
\end{disentanglement}

That is, a scheme $\cS$ with $N$ components has a model in $\cC$.
We can see that \cref{dis:functor} defines disentanglement via the \emph{product of functors} (based on the product in the codomain category $\cC$), while \cref{dis:multifunctor} uses the \emph{product of domain categories} (based on the product in $\cCat$).
They have their own application scenarios, but due to space limits, we will not study \cref{dis:functor} and \cref{dis:multifunctor} further in this paper.

%%%%%%%%%%%%%%%%%%%%%%%%%%%%%%%%%%%%%%%%%%%%%%%%%%

\subsection{Product-preserving Functors}

Instead, let us consider a definition based on the \emph{product in the domain category} $\cS$, which could be more flexible:

\begin{disentanglement}[A natural transformation at a product]
\label{dis:nat}
Let $\cS$ be a category with binary products, $\cC$ be a category, and $F_X, F_Z: \cS \to \cC$ be functors.
A disentangled encoding process is a \emph{component} of a natural transformation $\phi: F_X \nat F_Z$ at a \emph{product}.
\end{disentanglement}

%%%%%%%%%%%%%%%%%%%%

Additionally, if the codomain category $\cC$ also has products, we can require that

\begin{disentanglementsub}
\label{dis:fzproduct}
$F_Z$ is \catterm{product-preserving}.
\end{disentanglementsub}

In other words, $F_Z$ should be a \catterm{cartesian (monoidal) functor}, so products and projections in $\cS$ are mapped to products and projections in $\cC$.
An example is shown below:
\begin{equation}
\begin{tikzcd}[column sep=3.2em, row sep=.6em]
&
\pone
\arrow[ldddd, functor, end anchor={[xshift=-2mm]north}]
\arrow[rddd, functor]
\arrow[loop, "a"', distance=1em, out=-10, in=10, start anchor={[yshift=-.5mm]east}, end anchor={[yshift=.5mm]east}]
&
\\
&
\pone \times \ptwo
\arrow[u]
\arrow[d]
\arrow[lddd, functor, end anchor={north}]
\arrow[rddd, functor]
\arrow[loop, "a \times b"' fill=white, distance=1em, out=-10, in=10, start anchor={[yshift=-.5mm]east}, end anchor={[yshift=.5mm]east}]
\\
&
\ptwo
\arrow[ldd, functor, end anchor={[xshift=2mm]north}]
\arrow[rddd, functor]
\arrow[loop, "b"', distance=1em, out=-10, in=10, start anchor={[yshift=-.5mm]east}, end anchor={[yshift=.5mm]east}]
\\[-1.4em]
&&
Z_1
\arrow[loop, "a_Z"', distance=1em, out=-10, in=10]
\\
X
\arrow[rr, "f \defeq \phi_{\pone \times \ptwo}" {description, inner sep=1pt}]
\arrow[rru, "\phi_\pone" {description, inner sep=1pt}]
\arrow[rrd, "\phi_\ptwo" {description, inner sep=1pt}]
\arrow[loop, "a_X"', distance=1em, out=125, in=145]
\arrow[loop, "(a \times b)_X"', distance=1em, out=170, in=190]
\arrow[loop, "b_X"', distance=1em, out=215, in=235]
&&
Z
\arrow[u]
\arrow[d]
\arrow[loop, "\substack{(a \times b)_Z\\=a_Z \times b_Z}"', distance=1em, out=-10, in=10]
\\
&&
Z_2
\arrow[loop, "b_Z"', distance=1em, out=-10, in=10]
\end{tikzcd}
\end{equation}
We can see that two $\cS$-objects $\pone$ and $\ptwo$ have a product $\pone \times \ptwo$.
$F_Z$ preserves products so $(a \times b)_Z = a_Z \times b_Z$.
A disentangled encoding process $f \defeq \phi_{\pone \times \ptwo}$ is a component of a natural transformation $\phi$ at a product $\pone \times \ptwo$.
Note that $X$ is not necessarily a product but its endomorphisms can have a product structure \citep{higgins2018towards}.

%%%%%%%%%%%%%%%%%%%%

Next, let us check what the counterpart of \emph{modularity} is in the context of natural transformations.
What we will do here is essentially the same as what we showed in \cref{ssec:enc}.
Again, it is natural to consider a code generating process $\mu: F_Y \nat F_Z$ in $[\cS, \cC]$, and we have a counterpart of \cref{asm:yproduct} as follows:

\begin{assumption}
\label{asm:fyproduct}
$F_Y$ is product-preserving.
\end{assumption}

Then, we can simply say that a modular encoder $\mu$ is a \emph{natural transformation between product-preserving functors}.
Even more, we can prove the following property:

\begin{proposition}
\label{prop:muproduct}
$\forall \pone, \ptwo \in \cS.\; \mu_{\pone \times \ptwo} = \mu_\pone \times \mu_\ptwo$.
\end{proposition}

The reader should compare \cref{dis:fzproduct}, \cref{asm:fyproduct}, and \cref{prop:muproduct} with \cref{dis:mproduct}.

%%%%%%%%%%%%%%%%%%%%

The following commutative diagram encompasses all the requirements (cf.~\cref{prop:m}):

\begin{equation}
\begin{tikzcd}[column sep=2em, row sep=.5em, labels={inner sep=1pt}]
Y
\arrow[rr, "\mu_A", pos=.7]
\arrow[dd, "p_i"', pos=.2]
\arrow[rd, "a_Y" description]
&&[1em]
Z
\arrow[dd, "p_i", pos=.2]
\\
&
Y
\arrow[rr, "\mu_A", pos=.3, crossing over]
&&
Z
\arrow[dd, "p_i", pos=.2]
\arrow[from=lu, "a_Z" description]
\\
Y_i
\arrow[rr, "\mu_{A_i}"', pos=.7]
\arrow[rd, "{a_i}_Y" description]
&&
Z_i \arrow[rd, "{a_i}_Z" description]
\\
&
Y_i
\arrow[rr, "\mu_{A_i}"', pos=.3]
\arrow[from=uu, "p_i"', pos=.2, crossing over]
&&
Z_i
\end{tikzcd}
\end{equation}
The three axes correspond to (i) product, (ii) endomorphism, and (iii) natural transformation, respectively.

%%%%%%%%%%%%%%%%%%%%

Up to this point, our definition includes the one proposed by \citet{higgins2018towards} as a special case.
The reader may have noticed that there is only a counterpart of modularity \cref{dis:mproduct} but not explicitness \cref{dis:mmono}.
Without the requirement, we may encounter the same failure case:

\begin{example}[Constant]
The \catterm{constant functor} $\Delta 1$ satisfies \cref{dis:fzproduct} with a natural transformation $e_Y: Y \to 1$.
\end{example}

%%%%%%%%%%%%%%%%%%%%

To patch this, one way is to require that

\begin{disentanglementsub}
\label{dis:fzfaithful}
$F_Z$ is \catterm{faithful}.
\end{disentanglementsub}

This means that $F_Z$ is injective on morphisms for each pair of $\cS$-objects.
We need to rule out unfaithful models of a scheme lest we end up with uninformative representations.
This requirement also tells us some basic properties the codes $Z$ should have such as the minimal size or dimension, depending on the choice of the scheme $\cS$.

%%%%%%%%%%%%%%%%%%%%

On the other hand, the exact counterpart of explicitness \cref{dis:mmono} is as follows:

\begin{disentanglementsub}
\label{dis:mumono}
$\mu$ is a split monomorphism.
\end{disentanglementsub}

\cref{dis:mumono} is a stronger notion when $F_Y$ is also faithful:
\begin{theorem}
\label{thm:monofaithful}
$\cref{dis:mumono} \limp \cref{dis:fzfaithful}$.
\end{theorem}

%%%%%%%%%%%%%%%%%%%%

As a final note, we point out that \cref{dis:nat} is more flexible because it is not limited to endomorphisms:

\begin{example}[Binary operation]
Let $\pone$ be an $\cS$-object (which itself can be a product) and $c: \pone \times \pone \to \pone$ an $\cS$-morphism.
The following diagram describes how binary operations can exhibit disentanglement:
\begin{equation}
\begin{tikzcd}[column sep=1.5em, row sep=1em]
&
\pone \times \pone
\arrow[d, "c"]
\arrow[ldd, functor]
\arrow[rdd, functor]
\arrow[loop, "a \times b"', distance=1em, out=-10, in=10]
\\
&
\pone
\arrow[ldd, functor]
\arrow[rdd, functor]
\\[-.8em]
X \times X
\arrow[d, "c_X"']
\arrow[rr, "f \times f = \phi_{\pone \times \pone}"']
\arrow[loop, "a_X \times b_X"', distance=1em, out=170, in=190]
&&
Z \times Z
\arrow[d, "c_Z"]
\arrow[loop, "a_Z \times b_Z"', distance=1em, out=-10, in=10]
\\
X
\arrow[rr, "f \defeq \phi_\pone"']
&&
Z
\end{tikzcd}
\end{equation}
Regarding $c \circ (a \times b)$, the functoriality and naturality lead to the following requirement:
\begin{equation*}
f(c_X(a_X(x_1), b_X(x_2))) = c_Z(a_Z(f(x_1)), b_Z(f(x_2))).
\end{equation*}
\end{example}

This formulation is particularly useful when dealing with multiple instances or heterogeneous inputs \citep{gatys2016image, liu2018leveraging}.
Further investigation is left for future work.

%%%%%%%%%%%%%%%%%%%%

In summary, we showed that seemingly distinct approaches to disentanglement \citep{ridgeway2018learning, higgins2018towards} can be described by the same abstract language, and their underlying mechanisms (e.g., modularity and product-preserving action) are essentially the same.
The core is the \emph{cartesian product} of sets, functions, algebras, actions, objects, morphisms, categories, and functors.

%% file: sections/relation.tex
\epigraph{The Cartesian product of sets is not cartesian in $\cRel$.}

In this section, we present an example of (non-cartesian) monoidal products using $\cRel$, the category of sets and relations \citep{patterson2017knowledge}.

We may want to work with relations instead of functions if we need to consider
(i) unannotated factors,
(ii) multiple observations for the same factor, or
(iii) only a subset of all combinations of factors.
Besides, $\cRel$ serves as a bridge between functions and probabilities, which will be discussed in the next section.

%%%%%%%%%%%%%%%%%%%%

To characterize $\cRel$, it is convenient to consider it as the \catterm{Kleisli category} of the \catterm{powerset monad} $P$ in $\cSet$:
\begin{equation}
\cRel \defeq \cSet_P.
\end{equation}
The powerset monad $P$ sends a set $A$ to its powerset $PA$ and a function $f: A \to B$ to a set function $Pf: PA \to PB$.\footnote{The map on morphisms is the polymorphic function \lstinline{fmap :: Functor f => (a -> b) -> f a -> f b} in Haskell.}
Its Kleisli category $\cRel$ has relations $A \klto B$ as the \catterm{Kleisli morphisms}, which are precisely set-valued functions $A \to PB$.
The composition is the \catterm{Kleisli composition} $g \KleisliL f$,\footnote{This is the \say{left fish} operator in Haskell: \lstinline{(<=<) :: Monad m => (b -> m c) -> (a -> m b) -> a -> m c}.} given by the \emph{relation composition}.

Relations arise naturally in practice.
For example, if we have a \newterm{labeling process} $l: X \to Y$, which is a function in $\cSet$, a data generating process $g: Y \klto X \defeq l^*$ can be defined as its \catterm{inverse image}, which is not a function anymore but a relation in $\cRel$.
Then, $g$ now can map a factor to multiple observations or the empty set.

%%%%%%%%%%%%%%%%%%%%%%%%%%%%%%%%%%%%%%%%%%%%%%%%%%

\subsection{Monoidal Product of Relations}

Next, let us examine the product structures in $\cRel$.
We point out the following three important facts about $\cRel$:
\begin{itemize}
\item $\cRel$ is cartesian and cocartesian, with both the product and coproduct given by the \emph{disjoint union} of sets $A \oplus B$.
\item $\cRel$ is \catterm{monoidal closed}, with both the monoidal product and \catterm{internal hom} given by the \emph{Cartesian product} of sets $A \otimes B$ and the monoidal unit given by the singleton $\singleton$.
% \item $\cRel$ is \catterm{distributive}, i.e., $\otimes$ distributes over $\oplus$.
\item $\cRel$ is \catterm{pointed}, with the \catterm{zero object} (an object that is both initial and terminal) given by the empty set $\varnothing$.
\end{itemize}
That is, in $\cRel$, the Cartesian product of sets is \emph{monoidal}, but confusingly, \emph{not cartesian}.
So a relation $A \klto B \otimes C$ to a Cartesian product of two sets is more than just a pair of relations $A \klto B$ and $A \klto C$.
On the other hand, the monoidal product/internal hom $\otimes$ gives us an isomorphism between hom-sets:
\begin{equation}
\label{eq:relhom}
\Hom(A \otimes B, C) \iso \Hom(A, B \otimes C),
\end{equation}
which leads to the following example:
\begin{equation}
\begin{tikzcd}[column sep=1em, row sep=0, nodes={inner sep=1pt}]
(a, 0)
\arrow[rd, start anchor=east, end anchor=west, dash]
\arrow[rdd, start anchor=east, end anchor=west, dash]
\\
(a, 1)
\arrow[rd, start anchor=east, end anchor=west, dash]
&
x
\\
(b, 0)
&
y
\\
(b, 1)
\arrow[ruu, start anchor=east, end anchor=west, dash]
\\[.3em]
A \otimes B
\arrow[r, rightsquigarrow]
&
C
\end{tikzcd}
\;\;
\iso
\;\;
\begin{tikzcd}[column sep=1em, row sep=0, nodes={inner sep=1pt}]
&
(0, x)
\\
a
\arrow[ru, start anchor=east, end anchor=west, dash]
\arrow[r, start anchor=east, end anchor=west, dash]
\arrow[rdd, start anchor=east, end anchor=west, dash]
&
(0, y)
\\
b
\arrow[r, start anchor=east, end anchor=west, dash]
&
(1, x)
\\
&
(1, y)
\\[.3em]
A
\arrow[r, rightsquigarrow]
&
B \otimes C
\end{tikzcd}
\;\;
\ncong
\;\;
\begin{tikzcd}[column sep=1em, row sep=0, nodes={inner sep=1pt}]
a
\arrow[r, start anchor=east, end anchor=west, dash]
\arrow[rd, start anchor=east, end anchor=west, dash]
&
0
\\
b
\arrow[r, start anchor=east, end anchor=west, dash]
&
1
\\[.3em]
A
\arrow[r, rightsquigarrow]
&
B
\end{tikzcd}
\;
\otimes
\;
\begin{tikzcd}[column sep=1em, row sep=0, nodes={inner sep=1pt}]
a
\arrow[r, start anchor=east, end anchor=west, dash]
\arrow[rd, start anchor=east, end anchor=west, dash]
&
x
\\
b
\arrow[ru, start anchor=east, end anchor=west, dash]
&
y
\\[.3em]
A
\arrow[r, rightsquigarrow]
&
C
\end{tikzcd}
\end{equation}

%%%%%%%%%%%%%%%%%%%%

$\cRel$ is an example of how the cartesian product $\oplus$ is not an appropriate concept for disentanglement, while a suitable one $\otimes$ only has a monoidal structure.
The monoidal unit $\singleton$ is different from the terminal object $\varnothing$, so $\cRel$ is not even semicartesian.
Although we still can define the \say{duplicating} and \say{deleting} operations \citep[Section 3.3]{patterson2017knowledge}, they do not behave as nicely as those diagonal and terminal morphisms in $\cSet$ because of their non-naturality.

%%%%%%%%%%%%%%%%%%%%

Then, how can we characterize disentanglement?
At least, we still have a counterpart of disentanglement \cref{dis:product}:

\begin{disentanglement}[A morphism to a monoidal product]
\label{dis:monoidal}
In a symmetric monoidal category $\cC$, a disentangled encoding process is a morphism $f: X \to Z$ to a \emph{monoidal product} $Z \defeq \bigotimes_{i=1}^N Z_i$.
\end{disentanglement}

Further, we can extend the definition of modularity \cref{dis:mproduct}:

\begin{disentanglementsub}
\label{dis:mmonoidal}
$m = \bigotimes_{i=1}^N (m_{i,i}: Y_i \to Z_i)$.
\end{disentanglementsub}

So, \cref{dis:product} and \cref{dis:mproduct} are special cases of \cref{dis:monoidal} and \cref{dis:mmonoidal} for a cartesian category.
However, without projections, \cref{dis:mmonoidal} is more difficult to verify than \cref{dis:mproduct}.

%%%%%%%%%%%%%%%%%%%%

Then, how can we resolve this?
One way is to restrict our attention to \emph{right-unique relations}, i.e., \emph{partial functions}, so duplication behaves nicely (\cref{eq:diagonal}), but it means that there is at most one observation for each factor.
We can also focus on \emph{left-total relations}, i.e., \emph{multivalued functions}, so deletion behaves nicely (\cref{eq:terminal}), but we need to assume that there is at least one observation for each factor \citep[Example 2.6]{fritz2020synthetic}.
If we want both, then we will end up with $\cSet$ --- a cartesian subcategory of $\cRel$.
Despite its many good properties, $\cSet$ might be too restrictive if we want to incorporate uncertainty in disentanglement.
Later we will see that a \emph{semicartesian category with (not necessarily natural) diagonals} might be a balanced choice, which provides a rich collection of operations to characterize disentanglement.

%%%%%%%%%%%%%%%%%%%%%%%%%%%%%%%%%%%%%%%%%%%%%%%%%%

\subsection{Functor Category, Revisited}

Before moving on to the next section, \say{can we change from $\cRel$ to $[\cS, \cRel]$?} we have to ask.
First, the fact that $[\cS, \cC]$ has a pointwise monoidal structure derived from $\cC$ tells us that \cref{dis:functor} generalizes to the \emph{functor monoidal product} straightforwardly.
Second, \cref{dis:fzproduct} is a special case of the following requirement for a cartesian category:

\begin{disentanglementsubalt}{dis:fzproduct}
$F_Z$ is a \catterm{monoidal functor}.
\end{disentanglementsubalt}

\citet{higgins2018towards} mainly worked with the direct sum $\oplus$ (direct product $\times$) of vector spaces and briefly mentioned the tensor product $\otimes$.
We remind the reader that their decisive difference is between the cartesian and monoidal products.

%% file: sections/probability.tex
\epigraph{We can copy/delete in a Markov category like $\cStoch$.}

Besides the algebraic approach \citep{higgins2018towards}, the probabilistic, statistical, and information-theoretic methods \citep{higgins2017betavae, chen2018isolating, kumar2018variational, suter2019robustly, do2020theory} are perhaps the most popular tools for disentangled representation learning.
In this section, we outline the essential operations required for characterizing disentanglement of stochastic maps.

The structure is similar to that of $\cRel$: the category $\cStoch$ of measurable spaces and stochastic maps (Markov kernels) is the Kleisli category of the \catterm{Giry monad} $P$ in the category $\cMeas$ of measurable spaces and measurable functions:
\begin{equation}
\cStoch \defeq \cMeas_P.
\end{equation}
The Giry monad $P$ sends a measurable set $A$ to the set $PA$ of probability measures on $A$ and a measurable function $f: A \to B$ to its pushforward $f_*: PA \to PB$.
The Kleisli morphisms are stochastic maps $p(B|A)$, and the Kleisli composition $p(C|A) = p(C|B) \KleisliL p(B|A)$ is the \emph{Chapman--Kolmogorov equation} \citep{giry1982categorical}.

%%%%%%%%%%%%%%%%%%%%%%%%%%%%%%%%%%%%%%%%%%%%%%%%%%

\subsection{Joint Distribution and Conditional Independence}

Next, let us start by highlighting the impossibility result in \citet{locatello2019challenging}, which is essentially about the product structures of $\cStoch$.
It can be succinctly restated in the categorical language as

\begin{theorem}[{\citet[Theorem 1]{locatello2019challenging}}]
$\cStoch$ is not cartesian.
\end{theorem}

This theorem implies the following diagram (cf.~\cref{eq:product}):
\begin{equation}
\label{eq:marginal}
\begin{tikzcd}[column sep=2em, row sep=2em]
&
I
\arrow[d, "p"]
\arrow[ld, "p_1"']
\arrow[rd, "p_2"]
\\
Z_1
&
Z_1 \otimes Z_2
\arrow[l, "\pi_1"]
\arrow[r, "\pi_2"']
\arrow[loop, "f \neq \id_{Z_1 \otimes Z_2}"', distance=1em, out=260, in=280, pos=.85]
&
Z_2
\end{tikzcd}
\end{equation}
It means that a joint distribution $p(Z_1, Z_2)$ is not uniquely specified by its marginals $p_1(Z_1)$ and $p_2(Z_2)$.
\citet{locatello2019challenging} explicitly constructed a family of bijections $f: Z \to Z$ using the inverse transform sampling technique.

%%%%%%%%%%%%%%%%%%%%

Note the projection morphisms $\pi_1$ and $\pi_2$ used in \cref{eq:marginal}, which are available because $\cStoch$ is a Markov category \citep{fritz2020synthetic}.
A Markov category, roughly speaking, is a category in which every object $A$ is equipped with a \say{copy} $\cpy_A: A \to A \otimes A$ (not necessarily natural in $A$) and a \say{delete}  $\del_A: A \to I$ (natural in $A$) morphism satisfying some coherence conditions.
Therefore, all morphisms are \newterm{deletable} but only some are \newterm{copyable}, which allows for a sufficiently expressive category with enough operations to characterize disentanglement:

\begin{disentanglement}[A Markov kernel to a joint]
\label{dis:markov}
In a Markov category $\cC$, a disentangled encoding process is a \emph{Markov kernel} $f: X \to Z$ to a \emph{joint} $Z \defeq \bigotimes_{i=1}^N Z_i$.
\end{disentanglement}

%%%%%%%%%%%%%%%%%%%%

We point out that the \catterm{conditional independence} $A \ind B \condin C$ of a Markov kernel $p(A, B | C)$ \citep[Definition 12.12]{fritz2020synthetic} can be used to derive a prerequisite for the modularity of an encoder $m: Y \to Z$:

\begin{disentanglementsub}
\label{dis:mindout}
$\forall i \in [1..N].\; Z_i \ind Z_{\setminus i} \condin Y$.
\end{disentanglementsub}

%%%%%%%%%%%%%%%%%%%%

Let $m_i: Y \to Z_i \defeq \del_{Z_{\setminus i}} \compL m$ be the \catterm{marginalization} of $m$ over $Z_{\setminus i}$ and $\cpy_A^N: A \to A^{\otimes N}$ a \say{multiple copy} morphism.
We can prove that \cref{dis:mindout} is equivalent to the following equational identity (cf.~\cref{eq:product}):

\begin{disentanglementsub}
\label{dis:mproject}
$m = \parens{\bigotimes_{i=1}^N m_i} \compL \cpy_Y^N$.
\begin{equation*}
\begin{tikzpicture}
  \begin{pgfonlayer}{nodelayer}
    \node at (0, 0) {$=$};
    \node (1)  at (-1.5, -.4) {};
    \node (2)  at (-1.5,   0) {};
    \node (3)  at (-2  ,   0) {};
    \node (4)  at (-2  ,  .4) {};
    \node (5)  at (-1  ,   0) {};
    \node (6)  at (-1  ,  .4) {};
    \node (7)  at ( 2.5, -.7) {};
    \node (8)  at ( 2.5, -.5) [style=diagonal] {};
    \node (9)  at ( 1.5,   0) {};
    \node (10) at ( 3.5,   0) {};
    \node (11) at ( 1  ,   0) {};
    \node (12) at ( 1  ,  .7) {};
    \node (13) at ( 2  ,   0) {};
    \node (14) at ( 2  ,  .5) [style=diagonal] {};
    \node (15) at ( 3  ,   0) {};
    \node (16) at ( 3  ,  .5) [style=diagonal] {};
    \node (17) at ( 4  ,   0) {};
    \node (18) at ( 4  ,  .7) {};
    \node at (-1.5, -.6) {$Y$};
    \node at (-2  ,  .6) {$Z_1$};
    \node at (-1  ,  .6) {$Z_2$};
    \node at (-1.5,   0) [style=morphism, minimum width=1.6cm] {$m$};
    \node at ( 1.5,   0) [style=morphism, minimum width=1.6cm] {$m$};
    \node at ( 3.5,   0) [style=morphism, minimum width=1.6cm] {$m$};
  \end{pgfonlayer}
  \begin{pgfonlayer}{edgelayer}
    \draw (1.center) to (2.center);
    \draw (3.center) to (4.center);
    \draw (5.center) to (6.center);
    \draw (7.center) to (8.center);
    \draw [out=180, in=-90, looseness=1.2] (8.center) to (9.center);
    \draw [out=  0, in=-90, looseness=1.2] (8.center) to (10.center);
    \draw (11.center) to (12.center);
    \draw (13.center) to (14.center);
    \draw (15.center) to (16.center);
    \draw (17.center) to (18.center);
  \end{pgfonlayer}
\end{tikzpicture}
\end{equation*}
\end{disentanglementsub}

\begin{theorem}
\label{thm:indout}
$\cref{dis:mindout} \leqv \cref{dis:mproject}$.
\end{theorem}

We call an encoder satisfying \cref{dis:mproject} \newterm{projectable}.
This is a more fine-grained condition than the \emph{total correlation} used in \citet{chen2018isolating} because it is conditioned on the factors.

%%%%%%%%%%%%%%%%%%%%

With this precondition, we can finally define the modularity of a stochastic encoder:

\begin{disentanglementsub}
\label{dis:mindin}
$\forall i \in [1..N].\; \forall n_i: Y_i \to Y_{\setminus i}.\;$
$Z_i \ind Y_{\setminus i} \condin Y_i$.
\begin{equation*}
\begin{tikzpicture}
  \begin{pgfonlayer}{nodelayer}
    \node at (0, 0) {$=$};
    \node (1)  at (-1  , - .8) {};
    \node (2)  at (-1  , - .6) [style=diagonal] {};
    \node (3)  at (-2.5,  1.3) {};
    \node (4)  at (-1  ,   .4) [style=diagonal] {};
    \node (5)  at (-1.5,   .9) {};
    \node (6)  at (- .5,   .9) {};
    \node (7)  at (- .5,  1.3) {};
    \node (8)  at (-2.5,  1.3) {};
    \node (9)  at ( 2  , - .8) {};
    \node (10) at ( 2  , - .3) [style=diagonal] {};
    \node (11) at ( 2  , - .6) [style=diagonal] {};
    \node (12) at (  .5,   .7) {};
    \node (13) at ( 1.5,   .2) {};
    \node (14) at ( 2.5,   .2) {};
    \node (15) at ( 1.5,   .7) {};
    \node (16) at ( 2.5,  1.3) {};
    \node (17) at (  .5,  1.3) {};
    \node at (-1  , -1  ) {$Y_i$};
    \node at (-2.5,  1.5) {$Z_i$};
    \node at (- .5,  1.5) {$Y_{\setminus i}$};
    \node at ( 2  , -1  ) {$Y_i$};
    \node at (  .5,  1.5) {$Z_i$};
    \node at ( 2.5,  1.5) {$Y_{\setminus i}$};
    \node at (-2  ,  .9) [style=morphism, minimum width=1.6cm] {$m_i$};
    \node at (-1  , -.1) [style=morphism, minimum width=.6cm] {$n_i$};
    \node at ( 1  ,  .9) [style=morphism, minimum width=1.6cm] {$m_i$};
    \node at ( 1.5,  .2) [style=morphism, minimum width=.6cm] {$n_i$};
    \node at ( 2.5,  .2) [style=morphism, minimum width=.6cm] {$n_i$};
  \end{pgfonlayer}
  \begin{pgfonlayer}{edgelayer}
    \draw (1.center) to (2.center);
    \draw [out=180, in=-90, looseness=1.5] (2.center) to (3.center);
    \draw (2.center) to (4.center);
    \draw [out=180, in=-90, looseness=1.5] (4.center) to (5.center);
    \draw [out=  0, in=-90, looseness=1.5] (4.center) to (6.center);
    \draw (6.center) to (7.center);
    \draw (3.center) to (8.center);
    \draw (9.center) to (10.center);
    \draw [out=180, in=-90, looseness=1.5] (11.center) to (12.center);
    \draw [out=180, in=-90, looseness=1.5] (10.center) to (13.center);
    \draw [out=  0, in=-90, looseness=1.5] (10.center) to (14.center);
    \draw (13.center) to (15.center);
    \draw (14.center) to (16.center);
    \draw (12.center) to (17.center);
  \end{pgfonlayer}
\end{tikzpicture}
\end{equation*}
\end{disentanglementsub}

The reader may have noticed that this means that any $n_i: Y_i \to Y_{\setminus i}$ behaves like a \catterm{deterministic morphism} \citep[Definition 10.1]{fritz2020synthetic} when composed with $m_i: Y \to Z_i$.

%%%%%%%%%%%%%%%%%%%%

Why do we need this?
It is because, not like $\cRel$, where $A \otimes B \klto C$ is the same thing as $A \klto B \otimes C$ (\cref{eq:relhom}), in $\cStoch$, $\Hom(A, B \otimes C)$ is larger than $\Hom(A \otimes B, C)$.
We need a \say{probe} in $\Hom(A, B)$ to characterize how $C$ depends on $B$, and $n_i: Y_i \to Y_{\setminus i}$ serves as this probe.

Based on this, we revealed a common loophole in existing statistical approaches: if we use the mutual information between factors and codes to characterize disentanglement, the distribution of factors is assumed to be fixed \citep{chen2018isolating, li2020progressive, tokui2022disentanglement}.
However, the training and test distributions could be different \citep{trauble2021disentangled}, and the existing definitions may be too coarse-grained and insufficient in such a situation.

%%%%%%%%%%%%%%%%%%%%%%%%%%%%%%%%%%%%%%%%%%%%%%%%%%

\subsection{Structured Markov Kernels}

An important fact is that \emph{the category of functors to the subcategory of deterministic morphisms is again a Markov category} \citep[Section 7]{fritz2020synthetic}, so we can deal with \say{structured} Markov kernels.
We end our discussion with an example based on this fact, without any category theory jargon, to show what we can get from our formulation.

\begin{example}[{$[\N, \cSet_N]_{\det}$}]
A robot processing video feed of multiple objects should be able to
(i) identify objects,
(ii) understand that objects continue to exist even if they are occluded (\emph{object permanence}), and
(iii) track the movement of hidden objects (\emph{invisible displacement}).
All these abilities should not be affected by the shape or color of the objects.
\end{example}

With category theory, we can formulate such a complex task with ease because the components are \emph{compositional}.
See \cref{app:example} for a detailed explanation.

%% file: sections/limitations.tex
As an initial step toward categorical characterization of disentanglement, this work only focused on the definitions.
Many other important aspects of disentanglement were excluded, such as metrics, supervision, learnability, models, methods, and experimental validation.

With a clear understanding of the definitions in place, the immediate next step would be to find a systematic way to \emph{enrich a definition to a metric}.
We hypothesize that a good direction includes the following three steps:
\begin{itemize}
\item equality $\leadsto$ metric
\item universal quantification $\leadsto$ aggregation
\item existential quantification $\leadsto$ approximation
\end{itemize}

With a good metric, we can quantify how good a model is, even if it does not strictly satisfy a disentanglement definition.
For example, from \cref{thm:moddec}, we know that a modular and explicit encoder must have a modular decoder.
Given some modularity and explicitness metrics, we may want to know \emph{how much} the modularity and explicitness of an encoder imply the modularity of its decoder.

Other potential future directions include the studies of partial combination of factors (\cref{sec:relation}) and unknown projection (\cref{asm:zproduct}).
The relation between \cref{dis:functor}, \cref{dis:multifunctor}, and \cref{dis:nat} deserves further investigation.
How to formulate disentanglement in more complex learning problems, such as reinforcement learning, is also an interesting direction.
While we have obtained more results for cartesian categories due to their favorable properties, further theoretical analyses on the monoidal category case would be useful.

%% file: sections/conclusion.tex
In this work, we presented a meta-analysis of several definitions of disentanglement \citep{cohen2014learning, cohen2015transformation, ridgeway2018learning, eastwood2018framework, higgins2018towards, chen2018isolating} using \emph{category theory} as a unifying language.
We revealed that some seemingly distinct formulations are just different manifestations of the same structures, the \emph{cartesian products and monoidal products}, in different categories.
We argued that the modularity (\emph{product morphism}) and explicitness (\emph{split monomorphism}) should be the defining properties of disentanglement and introduced tools to analyze these properties in various settings, including equivariant maps (\emph{functor category}) and stochastic maps (\emph{Markov category}).
We also reinterpreted some existing results \citep{locatello2019challenging} and provided support to some arguments based on empirical evidence \citep{ridgeway2018learning, trauble2021disentangled}.
We hope our findings can help researchers choose the most appropriate definition of disentanglement for their specific task and consequently discover better metrics, models, methods, and algorithms for disentangled representation learning.

%% file: sections/acknowledgements.tex
We would like to thank Tobias Fritz for answering our questions about Markov categories.
We thank Wei Wang and Johannes Ackermann for their valuable feedback on the draft.
We also thank the anonymous reviewers for their useful comments and constructive suggestions. 
Finally, we would like to express our gratitude to all contributors to nLab, MathOverflow, and StackExchange for creating a sharing community.

YZ was supported by JSPS KAKENHI Grant Number 22J12703.
MS was supported by JST CREST Grant Number JPMJCR18A2.

%% file: sections/appendix.tex
\section{Additional Remarks}

%%%%%%%%%%%%%%%%%%%%%%%%%%%%%%%%%%%%%%%%%%%%%%%%%%

\subsection{Diagram}
\label{app:diagram}

We frequently use commutative diagrams \citep{awodey2006category} and string diagrams \citep{selinger2010survey, piedeleu2023introduction} as graphical calculus.

In a commutative diagram for a category, nodes are objects, arrows are morphisms, and paths are compositions of morphisms.
Any morphisms with the same domains and codomains are equal, i.e., any two paths starting from $A$ and ending with $B$ result in the same morphism.

In a string diagram for a symmetric monoidal category, rectangles are morphisms (from bottom to top), lines without rectangles are identity morphisms.
Juxtaposition of two morphisms means the monoidal product, and cross means the braiding.
Domains and codomains of morphisms are often omitted.

%%%%%%%%%%%%%%%%%%%%%%%%%%%%%%%%%%%%%%%%%%%%%%%%%%

\subsection{Compactness}
\label{app:compact}

Note that there could be two interpretations of compactness.
A non-compact encoder can mean:
\begin{enumerate}[(a), leftmargin=6mm]
\item We can recover $Y_j$ from $Z_i$; or
\item $Z_i$ itself is a product $Z_{i1} \times Z_{i2}$.
\end{enumerate}
We argue that (a) is what we want to avoid, while (b) is more or less harmless.
For example, we can decompose $\R^{10}$ into $\R^2 \times \R^3 \times \R^5$, where each component again can be decomposed into smaller parts.
However, sometimes this is beneficial: while embedding a cycle into a vector space, $\R^2$ may be a better choice than $\R$ because the embedding can be continuous.
In this work, we do not pay much attention to whether each code $Z_i$ is ``minimal''.

%%%%%%%%%%%%%%%%%%%%%%%%%%%%%%%%%%%%%%%%%%%%%%%%%%

\subsection{Functorial Semantics}

Can we formulate modularity in terms of functors and natural transformations?
The answer is yes, because the product, as a limit, can be defined via functors in the first place.
Here, we only give an alternative definition of \cref{dis:mproduct}:

\begin{disentanglementsubalt}{dis:mproduct}
$m$ is a natural transformation between functors from a \catterm{discrete category}.
\end{disentanglementsubalt}

\begin{equation}
\begin{tikzcd}[column sep=3em, row sep=.3em]
&
\pone
\arrow[ldd, functor]
\arrow[rdd, functor]
\\
&
\ptwo
\arrow[ldd, functor]
\arrow[rdd, functor]
\\[-.4em]
Y_1
\arrow[rr, "m_{1,1} \defeq m_\pone"']
&&
Z_1
\\
Y_2
\arrow[rr, "m_{2,2} \defeq m_\ptwo"']
&&
Z_2
\end{tikzcd}
\end{equation}
It shows that a modular encoder $m$ is merely a collection of component morphisms $m_{i,i}: Y_i \to Z_i$.
Nothing more, nothing less.

%%%%%%%%%%%%%%%%%%%%%%%%%%%%%%%%%%%%%%%%%%%%%%%%%%

\subsection{Commutativity and Irreducibility}

\citet{cohen2014learning} in their seminal paper used the \emph{irreducible} representations of compact \emph{commutative} Lie groups to define and study disentangled representations, while \citet{higgins2018towards} used the \emph{direct product} of groups.
Here, we briefly remark on the product, commutativity, and irreducibility.

First, let us keep it simple and only consider \emph{unital magma} --- a set with a unital binary operation.
If we have two unital magmas $(M, \compL_M, e_M)$ and $(N, \compL_N, e_N)$, we can define their product, denoted by $P = M \times N$, as the Cartesian product of their underlying sets equipped with a binary operation $\compL_P: (M \times N) \times (M \times N) \to (M \times N)$ given by
\begin{equation}
(m_1, n_1) \compL_P (m_2, n_2) \defeq (m_1 \compL_M m_2, n_1 \compL_N n_2),
\end{equation}
whose unit is $e_P \defeq (e_M, e_N)$.
We can see that the product is also a unital magma.

Then, we can find that every element $(m, n)$ in the product $P$ can be decomposed in two ways:
\begin{equation}
\begin{aligned}
  & (m, n)
\\
= & (m \compL_M e_M, e_N \compL_N n) = (m, e_N) \compL_P (e_M, n)
\\
= & (e_M \compL_M m, n \compL_N e_N) = (e_M, n) \compL_P (m, e_N),
\end{aligned}
\end{equation}
which can be expressed in string diagrams:
\begin{equation}
\begin{tikzpicture}
\node (A) [matrix, inner sep=0pt] {
  \draw (0, 0) -- (0, 2);
  \draw (1, 0) -- (1, 2);
  \node at (0, 1) [style=morphism, minimum width=.5cm] {$m$};
  \node at (1, 1) [style=morphism, minimum width=.5cm] {$n$};
  \\
};
\node (B) [matrix, inner sep=0pt, right=2cm of A, anchor=center] {
  \draw (0, 0) -- (0, 2);
  \draw (1, 0) -- (1, 2);
  \node at (0, 1.5) [style=morphism, minimum width=.5cm] {$m$};
  \node at (1,  .5) [style=morphism, minimum width=.5cm] {$n$};
  \\
};
\node (C) [matrix, inner sep=0pt, right=2cm of B, anchor=center] {
  \draw (0, 0) -- (0, 2);
  \draw (1, 0) -- (1, 2);
  \node at (0,  .5) [style=morphism, minimum width=.5cm] {$m$};
  \node at (1, 1.5) [style=morphism, minimum width=.5cm] {$n$};
  \\
};
\node at ($(A.east)!0.5!(B.west)$) {$=$};
\node at ($(B.east)!0.5!(C.west)$) {$=$};
\end{tikzpicture}
\end{equation}
We can identify $(m, e_N)$ as $m$ and $(e_M, n)$ as $n$ because of the unit magma isomorphisms:
\begin{align}
(M, \compL_M, e_M) & \iso (M \times \set{e_N}, \eval{\compL_P}{M \times \set{e_N}}, e_P),
\\
(N, \compL_N, e_N) & \iso (\set{e_M} \times N, \eval{\compL_P}{\set{e_M} \times N}, e_P).
\end{align}
From this perspective, as long as we can define a serial combination $\compL$ and its unit $e$ for each component, the product operation $\times$ allows us to combine elements from different components in parallel commutatively.
We can deal with one component at a time, and the order of the components does not matter.
However, note that the serial combination within a component may not be commutative, such as the 3D rotations $\mathrm{SO}(3)$ \citep{cohen2015transformation, higgins2018towards}.

This property may inspire us to \say{discover} disentangled components from observational data using commutativity:
we can find components such that elements from the same component are closed under composition, and elements from different components are commutative.

Such a decomposition may not be unique.
For example, consider $\R^2$ with addition $+$ (as a unital magma, a monoid, or a group).
$A = \set{(a, 0) \given a \in \R}$, $B = \set{(0, b) \given b \in \R}$, and $C = \set{(c, c) \given c \in \R}$ are all subalgebras of $\R^2$, and both $A \times B$ and $A \times C$ are isomorphic to $\R^2$ via the addition.

Learning the (potentially product) algebraic structure from data and determining an appropriate decomposition based on commutativity is an interesting research direction.

Besides, it is worth noting that \citet{cohen2014learning} identified a connection between irreducible representations and disentanglement, which is not covered in this work.
Furthermore, \citet{cohen2015transformation} made an insightful observation that irreducibility is also linked to the statistical dependency structure of the representation.
Using tools such as functor categories and Markov categories, we may obtain more fruitful results on the connection between algebraic and statistical properties of disentanglement.

%%%%%%%%%%%%%%%%%%%%%%%%%%%%%%%%%%%%%%%%%%%%%%%%%%

% \subsection{Kleisli Category}

% \begin{definition}[Kleisli category]
% For a \catterm{monad} $(T, \mu, \eta)$ in a category $\cC$, the Kleisli category $\cC_T$ is a category whose
% (i) objects are $\cC$-objects,
% (ii) morphisms are Kleisli morphisms $A \klto B$, which are $\cC$-morphisms of the form $A \to TB$, and
% (iii) the composition of $f: A \klto B$ and $g: B \klto C$ is the Kleisli composition:
% \begin{equation}
% g \KleisliL f: A \klto C := A \xto{f} TB \xto{Tg} TTC \xto{\mu_C} TC.
% \end{equation}
% \end{definition}

%%%%%%%%%%%%%%%%%%%%%%%%%%%%%%%%%%%%%%%%%%%%%%%%%%

\subsection{Probability}

Note that $\cMeas$ has finite products: $(A, \Sigma_A) \times (B, \Sigma_B) \defeq (A \times B, \Sigma_A \times \Sigma_B)$, where $\Sigma_A \times \Sigma_B$ is the coarsest $\sigma$-algebra such that two projections are measurable.

A useful construction is the category of probability measures and measure-preserving functions, which can be defined based on the concept of the \catterm{comma category}:
\begin{equation}
\cProb \defeq \cOne \incl \cStoch \comma \cMeas \to \cStoch.
\end{equation}

Concretely, $\cOne \incl \cStoch$ is the inclusion functor, and the functor $\cMeas \to \cStoch$ sends a measurable set $A$ to itself and a measurable function $f$ to its pushforward $f_*$.

$\cProb$ is a category whose objects are (isomorphic to) probability measures $(A, 1 \xto{p_A} PA)$, and morphisms $p_A \to p_B$ are measure-preserving functions $f: A \to B$ such that $p_B = f_* p_A$.
This category will be important when characterizing the metrics based on entropy and mutual information \citep{baez2011characterization}.

$\cMeas$, $\cStoch$, and $\cProb$ can be illustrated as follows:
\begin{equation}
\begin{adjustbox}{width=.85\linewidth}
\begin{tikzcd}[column sep=5em, row sep=6em, labels={description, inner sep=0pt}]
1
\arrow[rd, "p(A)"', pos=0.3, \red]
\arrow[rrd, "p(B)"', pos=0.31, \red]
\arrow[rrrd, "p(C)", pos=0.3, \red]
\arrow[r, "p(A)", rightsquigarrow, \green]
&
A
\arrow[rd, "p(B|A)", pos=0.75, crossing over, \blue]
\arrow[rrd, "p(C|A)", pos=0.5, \blue]
\arrow[r, "p(B|A)", rightsquigarrow, \green]
\arrow[rr, "p(C|A)", bend left, rightsquigarrow, \green]
&
B
\arrow[rd, "p(C|B)", pos=0.25, \blue]
\arrow[r, "p(C|B)", rightsquigarrow, \green]
&
C
\\
&
PA
\arrow[r, "f_*"', \yellow]
\arrow[rr, "g_* \circ f_*", bend right, \yellow]
&
PB
\arrow[r, "g_*"', \yellow]
\arrow[rd, "Pp(C|B)", crossing over]
&
PC
\\
&&&
PPC
\arrow[u, "\mu_C"]
\end{tikzcd}
\end{adjustbox}
\end{equation}

All arrows are morphisms in $\cMeas$;
\textcolor{\red}{red} arrows are \emph{objects} in $\cProb$;
\textcolor{\yellow}{yellow} arrows are morphisms in $\cProb$;
\textcolor{\green}{green} squiggly arrows represent morphisms in $\cStoch$, which are the same as \textcolor{\red}{red} or \textcolor{\blue}{blue} arrows.

The commutativity of \textcolor{\red}{red} and \textcolor{\yellow}{yellow} arrows indicates the composition of measure-preserving functions in $\cProb$; the commutativity of \textcolor{\blue}{blue} and black arrows indicates the Kleisli composition of stochastic maps in $\cStoch$.

As a side note, we can also use this construction to define the category of relations and relation-preserving functions \citep[Secion 3.3]{herrlich1990abstract}:
\begin{equation}
\cOne \incl \cRel \comma \cSet \to \cRel.
\end{equation}

%%%%%%%%%%%%%%%%%%%%%%%%%%%%%%%%%%%%%%%%%%%%%%%%%%

\subsection{Markov Category}

A Markov category \citep{fritz2020synthetic} is a symmetric monoidal category in which every object $A$ is equipped with a \catterm{commutative comonoid} structure given by a \catterm{comultiplication} $\cpy_A: A \to A \otimes A$ and a \catterm{counit} $\del_A: A \to I$, depicted in string diagrams as
\begin{equation}
\begin{tikzpicture}
  \begin{pgfonlayer}{nodelayer}
    \node at (0, 0) {$=$};
    \node at (-.8, 0) {$\cpy_A$};
    \node (1) at (.8, -.3) {};
    \node (2) at (.8, 0) [style=diagonal] {};
    \node (3) at (.4, .3) {};
    \node (4) at (1.2, .3) {};
  \end{pgfonlayer}
  \begin{pgfonlayer}{edgelayer}
    \draw (1.center) to (2.center);
    \draw [out=180, in=-90] (2.center) to (3.center);
    \draw [out=0, in=-90] (2.center) to (4.center);
  \end{pgfonlayer}
\end{tikzpicture}
\hspace{3em}
\begin{tikzpicture}
 \begin{pgfonlayer}{nodelayer}
    \node at (0, 0) {$=$};
    \node at (-.8, 0) {$\del_A$};
    \node (1) at (.8, -.3) {};
    \node (2) at (.8, .3) [style=diagonal] {};
  \end{pgfonlayer}
  \begin{pgfonlayer}{edgelayer}
    \draw (1.center) to (2.center);
  \end{pgfonlayer}
\end{tikzpicture}
\end{equation}
and \omitted{satisfying some compatibility conditions}.

% Please refer to the original paper \citep{fritz2020synthetic} for detailed characterizations and examples of Markov categories.

%%%%%%%%%%%%%%%%%%%%%%%%%%%%%%%%%%%%%%%%%%%%%%%%%%

\subsection{Conditional Independence}

\begin{definition}[Conditional independence {\citep[Definition 12.16]{fritz2020synthetic}}]
 A morphism $f: A \to X \otimes W  \otimes Y$ displays the \catterm{conditional independence} $X \ind Y \condout W \condin A$ if there are morphisms $g: A \to W$, $h: A \otimes W \to X$ and $k: W \otimes A \to Y$ such that
\begin{equation}
\begin{tikzpicture}
 \begin{pgfonlayer}{nodelayer}
  \node at (0, 0) {$=$};
  \node (1)  at (-2  , -1.3) {};
  \node (2)  at (-2  ,  0  ) {};
  \node (3)  at (-3  ,  0  ) {};
  \node (4)  at (-3  ,  1.3) {};
  \node (5)  at (-2  ,  0  ) {};
  \node (6)  at (-2  ,  1.3) {};
  \node (7)  at (-1  ,  0  ) {};
  \node (8)  at (-1  ,  1.3) {};
  \node (9)  at ( 2  , -1.3) {};
  \node (10) at ( 2  , - .9) [style=diagonal] {};
  \node (11) at ( 2  ,  .3) [style=diagonal] {};
  \node (12) at (  .8,  .9) {};
  \node (13) at ( 3.2,  .9) {};
  \node (14) at ( 1.2,  .9) {};
  \node (15) at ( 2.8,  .9) {};
  \node (16) at ( 1  ,  .9) {};
  \node (17) at ( 1  ,  1.3) {};
  \node (18) at ( 3  ,  .9) {};
  \node (19) at ( 3  ,  1.3) {};
  \node (20) at ( 2  ,  1.3) {};
  \node at (-2, -1.5) {$A$};
  \node at (-3,  1.5) {$X$};
  \node at (-2,  1.5) {$W$};
  \node at (-1,  1.5) {$Y$};
  \node at ( 2, -1.5) {$A$};
  \node at ( 1,  1.5) {$X$};
  \node at ( 2,  1.5) {$W$};
  \node at ( 3,  1.5) {$Y$};
  \node at (-2, 0) [style=morphism, minimum width=2.8cm] {$f$};
  \node at (2, -.3) [style=morphism, minimum width=.8cm] {$g$};
  \node at (1,  .9) [style=morphism, minimum width=.8cm] {$h$};
  \node at (3,  .9) [style=morphism, minimum width=.8cm] {$k$};
 \end{pgfonlayer}
 \begin{pgfonlayer}{edgelayer}
  \draw (1.center) to (2.center);
  \draw (3.center) to (4.center);
  \draw (5.center) to (6.center);
  \draw (7.center) to (8.center);
  \draw (9.center) to (10.center);
  \draw (10.center) to (11.center);
  \draw [out=180, in=-90, looseness=1.5] (10.center) to (12.center);
  \draw [out=0, in=-90, looseness=1.5] (10.center) to (13.center);
  \draw [out=180, in=-90, looseness=1.5] (11.center) to (14.center);
  \draw [out=0, in=-90, looseness=1.5] (11.center) to (15.center);
  \draw (16.center) to (17.center);
  \draw (18.center) to (19.center);
  \draw (11.center) to (20.center);
 \end{pgfonlayer}
\end{tikzpicture}
\end{equation}
\end{definition}
Two special cases are when $A = I$ we have $X \ind Y \condout W$ and when $W = I$ we have $X \ind Y \condin A$.

Another way to define the modularity of a stochastic encoder is as follows, which relies on the existence of some other morphisms (cf.~\cref{dis:const}):

\begin{disentanglementsub}
\label{dis:del}
$\forall i \in [1..N].\; m_i = m_{i,i} \otimes \del_{Y_{\setminus i}}$.
\begin{equation*}
\begin{tikzpicture}
  \begin{pgfonlayer}{nodelayer}
    \node at (0, 0) {$=$};
    \node (1) at (-2, -.4) {};
    \node (2) at (-2,  .4) {};
    \node (3) at (-1, -.4) {};
    \node (4) at (-1,   0) {};
    \node (5) at ( 1, -.4) {};
    \node (6) at ( 1,  .4) {};
    \node (7) at ( 2, -.4) {};
    \node (8) at ( 2,  .4) [style=diagonal] {};
    \node at (-2, -.6) {$Y_i$};
    \node at (-1, -.6) {$Y_{\setminus i}$};
    \node at (-2,  .6) {$Z_i$};
    \node at ( 1, -.6) {$Y_i$};
    \node at ( 2, -.6) {$Y_{\setminus i}$};
    \node at ( 1,  .6) {$Z_i$};
    \node at (-1.5, 0) [style=morphism, minimum width=1.6cm] {$m_i$};
    \node at ( 1  , 0) [style=morphism, minimum width=.8cm] {$m_{i,i}$};
  \end{pgfonlayer}
  \begin{pgfonlayer}{edgelayer}
    \draw (1.center) to (2.center);
    \draw (3.center) to (4.center);
    \draw (5.center) to (6.center);
    \draw (7.center) to (8.center);
  \end{pgfonlayer}
\end{tikzpicture}
\end{equation*}
\end{disentanglementsub}

This condition was also studied in \citet[Proposition 6.9]{cho2019disintegration}.
We can see that it is stronger than \cref{dis:mindin}:

\begin{theorem}
$\cref{dis:del} \limp \cref{dis:mindin}$.
\end{theorem}

However, it is not yet clear if they are equivalent.

%%%%%%%%%%%%%%%%%%%%%%%%%%%%%%%%%%%%%%%%%%%%%%%%%%

\clearpage
\section{Example}
\label{app:example}

Let us start from the category $\cSet$.
Consider the nonempty multiset monad $N$ in $\cSet$, which sends a set $A$ to $\N^A \setminus \varnothing$.
For example, the set $\set{a, b}$ is sent to
\begin{equation*}
\set{\set{(a, 1)}, \set{(a, 2)}, \dots, \set{(b, 1)}, \dots, \set{(a, 1), (b, 1)}, \dots}
\end{equation*}

The Kleisli category $\cSet_N$ of this monad consists of sets and multiset functions.
A multiset function $f: A \klto B$ outputs how many ways to get a target $b \in B$ from a source $a \in A$.
The composition of multiset functions is defined by the multiplication and sum of natural numbers.
This category is a Markov category.

A Markov category $\cC$ has a cartesian subcategory $\cC_{\det}$ of deterministic morphisms.
Given a small category $\cS$, the subcategory $[\cS, \cC]_{\det}$ of the functor category $[\cS, \cC]$, which consists of functors of the form $\cS \to \cC_{\det} \incl \cC$, is again a Markov category \citep[Section 7, notation slightly modified]{fritz2020synthetic}.
The category $[\cS, \cC]_{\det}$ contains deterministic diagrams of shape $\cS$ and stochastic maps between them that preserve the shape.

The set of natural numbers can be considered a single-object category $(\pone, \N, +, 0)$ with the numbers as morphisms and the addition as the composition.
The identity morphism $\id_\pone$ is the number $0$.

Based on these, let us consider the category $[\N, \cSet_N]_{\det}$.
This category contains sets equipped with endofunctions indexed by natural numbers as objects and multiset functions between these sets that preserve their endofunctions as morphisms.
A natural transformation to a constant functor (which maps all morphisms to the identity morphism) in this category means that no matter how the input changes with time, the count is invariant.
An example is shown bellow:

\begin{equation*}
\begin{tikzpicture}
\node (A) [matrix] {
  \node [rectangle, draw=black, thick, minimum size=2cm] {};
  \draw [fill=\red] (-.1, .6) circle (.3);
  \draw [fill=\green] (-.8, -.8) rectangle (.2, .1);
  \draw [fill=\red] (-.6, -.9) rectangle (-.2, -.5);
  \draw [fill=\blue] (.4, -.1) -- (0, -.9) -- (.8, -.9) -- cycle;
  \\
};
\node (B) [matrix, below=1.5cm of A, anchor=center] {
  \node [rectangle, draw=black, thick, minimum size=2cm] {};
  \draw [fill=\red] (-.1, .2) circle (.3);
  \draw [fill=\green] (-.8, -.8) rectangle (.2, .1);
  \draw [fill=\red] (-.6, -.9) rectangle (-.2, -.5);
  \draw [fill=\blue] (.4, -.1) -- (0, -.9) -- (.8, -.9) -- cycle;
  \\
};
\node (C) [matrix, below=1.5cm of B, anchor=center] {
  \node [rectangle, draw=black, thick, minimum size=2cm] {};
  \draw [fill=\red] (-.1, -.5) circle (.3);
  \draw [fill=\green] (-.8, -.8) rectangle (.2, .1);
  \draw [fill=\red] (-.6, -.9) rectangle (-.2, -.5);
  \draw [fill=\blue] (.4, -.1) -- (0, -.9) -- (.8, -.9) -- cycle;
  \\
};
\node (D) [right=.5cm of A] {$\set{(\text{red}, 2), (\text{green}, 1), (\text{blue}, 1)}$};
\node (E) [right=.5cm of B] {$\set{(\text{red}, 2), (\text{green}, 1), (\text{blue}, 1)}$};
\node (F) [right=.5cm of C] {$\set{(\text{red}, 2), (\text{green}, 1), (\text{blue}, 1)}$};
\path[commutative diagrams/.cd, every arrow, every label]
  (A) edge [commutative diagrams/mapsto] (B)
  (B) edge [commutative diagrams/mapsto] (C)
  (D) edge [commutative diagrams/mapsto] (E)
  (E) edge [commutative diagrams/mapsto] (F)
  (A) edge [commutative diagrams/mapsto] (D)
  (B) edge [commutative diagrams/mapsto] (E)
  (C) edge [commutative diagrams/mapsto] (F);
\end{tikzpicture}
\end{equation*}

If we want to characterize more complex behavior, we may simply change the source category $\N$ and define a proper category (possibly with a product structure) that encodes our requirements.
The extension is left for future work.

%%%%%%%%%%%%%%%%%%%%%%%%%%%%%%%%%%%%%%%%%%%%%%%%%%

% \clearpage
\section{Proofs}

\begin{proof}[\cref{prop:m}]
\begin{equation}
\begin{tikzcd}[row sep=3em]
Y_1
\arrow[d, "m_{1,1}"']
&
Y_1 \times Y_2
\arrow[l, "p_1"']
\arrow[r, "p_2"]
\arrow[d, "m_{1,1} \times m_{2,2}" description, unique morphism]
&
Y_2
\arrow[d, "m_{2,2}"]
\\
Z_1
&
Z_1 \times Z_2
\arrow[l, "p_1"]
\arrow[r, "p_2"']
&
Z_2
\end{tikzcd}
\end{equation}
\end{proof}

%%%%%%%%%%%%%%%%%%%%

\begin{proof}[\cref{prop:mii}]
\begin{equation}
\begin{tikzcd}[row sep=2em]
Y_1
\arrow[d, "\id_{Y_1}"']
&
Y_1 \times 1
\arrow[l, "p_1"']
\arrow[d, "\id_{Y_1} \times y_2"]
\\
Y_1
\arrow[d, "{m_{1,1}}"']
\arrow[ru, "\iso" description]
&
Y_1 \times Y_2
\arrow[l, "p_1"]
\arrow[d, "m"]
\\
Z_1
&
Z_1 \times Z_2
\arrow[l, "p_1"] 
\end{tikzcd}
\end{equation}
\end{proof}

%%%%%%%%%%%%%%%%%%%%

\cref{thm:exp} can be proven using the following lemma:

\begin{lemma}
\label{lem:exponential}
Let $f: A \times B \to C$ be a morphism from a product and $\widehat{f}: B \to C^A$ its exponential transpose.
Then, there exists a morphism $f': A \to C$ such that $f = f' \compL p_1$ if and only if the exponential transpose $\widehat{f}$ is a constant morphism.
\end{lemma}
\begin{proof}
Diagram chase:
\begin{equation}
% \begin{tikzcd}[column sep=2.5em, row sep=1.5em, arrows={labels={inner sep=1pt}}]
% &&
% 1
% \arrow[rd, "\widehat{f'}"]
% \\
% B
% \arrow[rrr, "\widehat{f}", unique morphism]
% \arrow[rru, unique morphism]
% &&&
% C^A
% \\[-1.5em]
% &&
% A \times 1
% \arrow[uu, "p_2"']
% \arrow[rdd, "\id_A \times \widehat{f'}", pos=.2]
% \\[-1.5em]
% &
% A
% \arrow[ru, "\iso", leftrightarrow]
% \arrow[dd, "f'"', pos=.6]
% \\[-1.5em]
% A \times B
% \arrow[rd, "f"']
% \arrow[rrr, "\id_A \times \widehat{f}"]
% \arrow[uuu, "p_2"]
% \arrow[ru, "p_1"]
% &&&
% A \times C^A
% \arrow[lld, "\epsilon_A"]
% \arrow[uuu, "p_2"']
% \\[.5em]
% &
% C
% \end{tikzcd}
\begin{tikzcd}[column sep=2.5em, row sep=1.5em, arrows={labels={inner sep=1pt}},
execute at end picture={
\draw [line width=.5pt, rounded corners=15, \red]
([xshift=20pt, yshift=-5pt] AB.center) --
([xshift=-35pt, yshift=-5pt] AAC.center) --
([xshift=0pt, yshift=7pt] C.center) --
cycle;
\draw [line width=.5pt, rounded corners=20, \green]
([xshift=30pt, yshift=5pt] B.center) --
([xshift=-0pt, yshift=-5pt] 1.center) --
([xshift=-15pt, yshift=5pt] AC.center) --
cycle;
\draw [line width=.5pt, rounded corners=10, \blue]
([xshift=15pt, yshift=-5pt] A1.center) --
([xshift=15pt, yshift=-20pt] 1.center) --
([xshift=-3pt, yshift=-10pt] AC.center) --
([xshift=-3pt, yshift=5pt] AAC.center) --
cycle;
\draw [line width=.5pt, rounded corners=10, \yellow]
([xshift=5pt, yshift=-5pt] A.center) --
([xshift=0pt, yshift=-8pt] A1.center) --
([xshift=-25pt, yshift=3pt] AAC.center) --
([xshift=5pt, yshift=12pt] C.center) --
cycle;
}]
&&
|[alias=1]|
1
\arrow[rd, "\widehat{f'}" \green]
\\
|[alias=B]|
B
\arrow[rrr, "\widehat{f}"', unique morphism]
\arrow[rru, unique morphism]
&&&
|[alias=AC]|
C^A
\\[-1.5em]
&&
|[alias=A1]|
A \times 1
\arrow[uu, "p_2"']
\arrow[rdd, "\id_A \times \widehat{f'}" \blue, pos=.1]
\arrow[dddd, "p_1", pos=.3]
\\[-1.5em]
&
|[alias=A]|
A
\arrow[ru, "\iso", leftrightarrow]
\arrow[dd, "f'"' \yellow, pos=.6]
\\[-1.5em]
|[alias=AB]|
A \times B
\arrow[rd, "f"']
\arrow[rrr, "\id_A \times \widehat{f}" \red]
\arrow[uuu, "p_2"]
\arrow[ru, "p_1"]
\arrow[ddd, "p_1"']
&&&
|[alias=AAC]|
A \times C^A
\arrow[lld, "\epsilon_A"]
\arrow[uuu, "p_2"']
\arrow[ddd, "p_1"]
\\[.5em]
&
|[alias=C]|
C
\\[-2em]
&&
A
\arrow[rd, "\id_A"]
\\
A
\arrow[rru, "\id_A"]
\arrow[rrr, "\id_A"]
&&&
A
\end{tikzcd}
\end{equation}

We need to use the following commutative diagrams:
\begin{enumerate*}[(i)]
\item \textcolor{\red}{red}: the universal property of the exponential object $C^A$ and the evaluation morphism $\epsilon_A$;
\item \textcolor{\green}{green}: the constant morphism $\widehat{f}$, which factors through the terminal object $1$ and defines the morphism $\widehat{f'}$;
\item \textcolor{\blue}{blue}: the product morphism $\id_A \times \widehat{f'}$; and
\item \textcolor{\yellow}{yellow}: the definition of $f'$.
\end{enumerate*}

It is easy to prove $\widehat{f}: B \to C^A$ is a constant morphism if $f = f' \compL p_1$.
Suppose $\widehat{f}: B \to C^A$ is a constant morphism, so it factors through the terminal object $1$.
We denote the morphism by $\widehat{f'}: 1 \to C^A$.
We can define $f': A \to C$ as $\epsilon_A \compL (\id_A \times \widehat{f'})$.
To prove $f = f' \compL p_1$, i.e., $f = \epsilon_A \compL (\id_A \times \widehat{f'}) \compL p_1$, we only need to show $\id_A \times \widehat{f} = (\id_A \times \widehat{f'}) \compL (\id_A \times e_B)$.
This triangle commutes because it is simply a product of the identity morphism $\id_A$ and the constant morphism $\widehat{f}$.
\end{proof}

%%%%%%%%%%%%%%%%%%%%

Alternatively, we can also characterize product morphisms using \catterm{pullback}.
Concretely, let $Y \times_{Y_i} Y$ be the pullback of the projections $p_i: Y \to Y_i$ and $\pi_1, \pi_2: Y \times_{Y_i} Y \to Y$ be the \catterm{pullback projections}.
In the category $\cSet$ of sets, the pullback $Y \times_{Y_i} Y = \set{(y, y') \in Y \times Y \given y_i = y'_i}$ is the set of pairs of factors whose $i$-th components are the same.
Then, $m$ is a product morphism if and only if $m_i \compL \pi_1 = m_i \compL \pi_2$, i.e., $m_i(y_i, y_{\setminus i}) = m_i(y_i, y'_{\setminus i})$.
This can be proven using the following lemma:

\begin{lemma}
\label{lem:pullback}
Let $f: A \times B \to C$ be a morphism from a product and $(A \times B) \times_A (A \times B)$ be the pullback of the projections $p_1: A \times B \to A$ with two pullback projections $\pi_1, \pi_2: (A \times B) \times_A (A \times B) \to A \times B$.
Then, there exists a morphism $f': A \to C$ such that $f = f' \compL p_1$ if and only if $f \compL \pi_1 = f \compL \pi_2$.
\end{lemma}
\begin{proof}
Diagram chase:
\begin{equation}
\begin{tikzcd}[column sep=0, row sep=2.5em, arrows={labels={inner sep=1pt}}]
&[1em]&[-2.5em]
A \times B
\arrow[rd, "p_1"']
\arrow[rrd, "f"]
\\
A
\arrow[r, "v" description, unique morphism]
\arrow[rru, "{\angles{\id_A, g}}", pos=.7]
\arrow[rrd, "{\angles{\id_A, g'}}"', pos=.7]
&
(A \times B) \times_A (A \times B)
\arrow[rd, "\pi_1"]
\arrow[ru, "\pi_2"']
&&[1.5em]
A
\arrow[r, "f'" description]
&[2em]
C
\\
A \times B
\arrow[ru, "u" description, unique morphism]
\arrow[rr, "\id_{A \times B}"']
\arrow[u, "p_1"]
&&
A \times B
\arrow[ru, "p_1"]
\arrow[rru, "f"']
\end{tikzcd}
\end{equation}

Suppose that $f = f' \compL p_1$.
Because the pullback rectangle commutes, $p_1 \compL \pi_1 = p_1 \compL \pi_2$, it is easy to show that $f \compL \pi_1 = f' \compL p_1 \compL \pi_1 = f' \compL p_1 \compL \pi_2 = f \compL \pi_2$.

Now suppose that $f \compL \pi_1 = f \compL \pi_2$.
We define $f': A \to C$ as $f \compL \angles{\id_A, g}$ for an arbitrary morphism $g: A \to B$.
To prove $f = f' \compL p_1$, we can consider two morphisms $\id_{A \times B}$ and $\angles{\id_A, g} \compL p_1$.
Because they complete the commutative diagram of the pullback $(A \times B) \times_A (A \times B)$, $p_1 \compL \id_{A \times B} = p_1 \compL \angles{\id_A, g} \compL p_1 = p_1$, there exists a unique morphism $u: A \times B \to (A \times B) \times_A (A \times B)$ such that $\pi_1 \compL u = \id_{A \times B}$ and $\pi_2 \compL u = \angles{\id_A, g} \compL p_1$.
We can now chase the diagram to show that $f = f \compL \id_{A \times B} = f \compL \pi_1 \compL u = f \compL \pi_2 \compL u = f \compL \angles{\id_A, g} \compL p_1 = f' \compL p_1$.

To prove that this construction does not depend on specific choice of $g: A \to B$, let us consider two morphisms $g, g': A \to B$.
Because $\angles{\id_A, g}$ and $\angles{\id_A, g'}$ complete the commutative diagram of the pullback, there exists a unique morphism $v: A \to (A \times B) \times_A (A \times B)$ such that $\pi_1 \compL v = \angles{\id_A, g'}$ and $\pi_2 \compL v = \angles{\id_A, g}$.
Then, $f \compL \angles{\id_A, g} = f \compL \pi_2 \compL v = f \compL \pi_1 \compL v = f \compL \angles{\id_A, g'}$, which shows that $f' = f \compL \angles{\id_A, g}$ is independent of the choice of $g: A \to B$.
\end{proof}

Based on this, we can obtain the following diagram:

\begin{equation}
\begin{tikzcd}[column sep=4em, row sep=2em]
Y \times_{Y_i} Y
\arrow[r]
\arrow[dd]
\arrow[rrd, unique morphism]
&
Y
\arrow[dd, "p_i"]
\arrow[rrd, "m"]
\\
&&[-3em]
Z \times_{Z_i} Z
\arrow[r]
\arrow[dd]
&
Z
\arrow[dd, "p_i"]
\\[-1em]
Y
\arrow[r, "p_i"']
\arrow[rrd, "m"']
&
Y_i
\arrow[rrd, "m_{ii}"]
&
\\
&&
Z
\arrow[r, "p_i"']
&
Z_i
\end{tikzcd}
\end{equation}

Both \cref{lem:exponential,lem:pullback} show that there are alternative ways to characterize \say{invariance}, without a group theoretical formulation.

%%%%%%%%%%%%%%%%%%%%

\begin{proof}[\cref{thm:moddec}]
\begin{equation}
\begin{tikzcd}[row sep=2em]
Y_1
\arrow[d, "\id_{Y_1}"']
&
Y_1 \times Y_2
\arrow[d, "\id_{Y_1 \times Y_2}"]
\arrow[l, "p_1"]
\\
Y_1
\arrow[d, "{m_{1,1}}"']
&
Y_1 \times Y_2
\arrow[d, "m"]
\arrow[l, "p_1"]
\\
Z_1
\arrow[rd, "\iso" description]
\arrow[uu, "{h_{1,1}}", bend left=80, start anchor=west, end anchor=west]
&
Z_1 \times Z_2
\arrow[l, "p_1"]
\arrow[uu, "h"', bend right=80, start anchor=east, end anchor=east]
\\
Z_1
\arrow[u, "\id_{Z_1}"]
&
Z_1 \times 1
\arrow[u, "\id_{Z_1} \times z_2"']
\arrow[l, "p_1"]
\end{tikzcd}
\end{equation}
\end{proof}

%%%%%%%%%%%%%%%%%%%%

\begin{proof}[\cref{prop:muproduct}]
Let $F, G: \cC \to \cD$ be product preserving functors.
\begin{equation}
\begin{adjustbox}{width=.88\linewidth}
\begin{tikzcd}[column sep=2em, row sep=4em]
&&
A
\arrow[lldd, functor]
\arrow[rrdd, functor]
&
A \times B
\arrow[l, "p_1"']
\arrow[r, "p_2"]
\arrow[ld, functor]
\arrow[rd, functor]
&
B
\arrow[lldd, functor]
\arrow[rrdd, functor]
\\
&&
F(A \times B)
\arrow[lld, "F p_1"']
\arrow[ld, unique morphism, shift right]
\arrow[d, "F p_2"]
\arrow[rr, "\alpha_{A \times B}"', bend right, \red]
&&
G(A \times B)
\arrow[d, "G p_1"']
\arrow[rd, unique morphism, shift left]
\arrow[rrd, "G p_2"]
&&
\\
FA
\arrow[rrrr, "\alpha_A"', bend right]
&
FA \times FB
\arrow[l, "p_1"]
\arrow[r, "p_2"']
\arrow[ru, shift right, \red]
\arrow[rrrr, "\alpha_A \times \alpha_B"', bend right, \red]
&
FB
\arrow[rrrr, "\alpha_B"', bend right]
&&
GA
&
GA \times GB
\arrow[l, "p_1"]
\arrow[r, "p_2"']
\arrow[lu, shift left, \red]
&
GB
\end{tikzcd}
\end{adjustbox}
\end{equation}
\end{proof}

%%%%%%%%%%%%%%%%%%%%

\newpage
\begin{proof}[\cref{thm:monofaithful}]
Let $F, G: \cC \to \cD$ be functors, $\alpha: F \nat G$ be a natural transformation.
\begin{equation}
\begin{tikzcd}[row sep=3em]
FA
\arrow[r, "\alpha_A"]
\arrow[d, "{Fp, Fq}"']
&
GA
\arrow[d, "{Gp, Gq}"]
\\
FB
\arrow[r, "\alpha_B"']
&
GB
\end{tikzcd}
\end{equation}

We have the following reasoning:
\begin{itemize}
\item $F$ is not faithful: $\exists p \neq q.\; Fp = Fq$
\item $\alpha$ is natural: $Fp = Fq \limp Gp \compL \alpha_A = Gq \compL \alpha_A$
\item $\alpha$ is epic: $Gp \compL \alpha_A = Gq \compL \alpha_A \limp Gp = Gq$
\end{itemize}
Then,
\begin{equation}
F \text{ is not faithful} \lcon \alpha \text{ is epic} \limp G \text{ is not faithful}.
\end{equation}
Or equivalently,
\begin{equation}
\alpha \text{ is epic} \limp (G \text{ is faithful} \limp F \text{ is faithful}).
\end{equation}

Similarly,
\begin{itemize}
\item $G$ is not faithful: $\exists p \neq q.\; Gp = Gq$
\item $\alpha$ is natural: $Gp = Gq \limp \alpha_B \compL Fp = \alpha_B \compL Fq$
\item $\alpha$ is monic: $\alpha_B \compL Fp = \alpha_B \compL Fq \limp Fp = Fq$
\end{itemize}
Then,
\begin{equation}
G \text{ is not faithful} \lcon \alpha \text{ is monic} \limp F \text{ is not faithful}.
\end{equation}
Or equivalently,
\begin{equation}
\alpha \text{ is monic} \limp (F \text{ is faithful} \limp G \text{ is faithful}).
\end{equation}
\end{proof}

%%%%%%%%%%%%%%%%%%%%

\begin{proof}[\cref{thm:indout}]
When $N = 2$, \cref{dis:mproject} is the definition of \cref{dis:mindout} \citep[Lemma 12.11]{fritz2020synthetic}.
When $N > 2$, we can apply this equation recursively.
\begin{equation}
\begin{adjustbox}{width=.8\linewidth}
\begin{tikzpicture}
  \begin{pgfonlayer}{nodelayer}
    \node at (0, 0) {$=$};
    \node (1)  at (-1.5, -.6) {};
    \node (2)  at (-1.5,   0) {};
    \node (3)  at (-2  ,   0) {};
    \node (4)  at (-2  ,  .6) {};
    \node (5)  at (-1.5,   0) {};
    \node (6)  at (-1.5,  .6) {};
    \node (7)  at (-1  ,   0) {};
    \node (8)  at (-1  ,  .6) {};
    \node (9)  at ( 3.5, - 1) {};
    \node (10) at ( 3.5, -.6) [style=diagonal] {};
    \node (11) at ( 1  ,   0) {};
    \node (12) at ( 1.5,   0) {};
    \node (13) at ( 2  ,   0) {};
    \node (14) at ( 1  ,  .8) {};
    \node (15) at ( 1.5,  .6) [style=diagonal] {};
    \node (16) at ( 2  ,  .6) [style=diagonal] {};
    \node (17) at ( 4.5,   0) [style=diagonal] {};
    \node (18) at ( 3  ,  .6) {};
    \node (19) at ( 3.5,  .6) {};
    \node (20) at ( 4  ,  .6) {};
    \node (21) at ( 3  , 1.2) [style=diagonal] {};
    \node (22) at ( 3.5, 1.4) {};
    \node (23) at ( 4  , 1.2) [style=diagonal] {};
    \node (24) at ( 5  ,  .6) {};
    \node (25) at ( 5.5,  .6) {};
    \node (26) at ( 6  ,  .6) {};
    \node (27) at ( 5  , 1.2) [style=diagonal] {};
    \node (28) at ( 5.5, 1.2) [style=diagonal] {};
    \node (29) at ( 6  , 1.4) {};
    \node at (-1.5, -.8) {$Y$};
    \node at (-2  ,  .8) {$Z_1$};
    \node at (-1.5,  .8) {$Z_2$};
    \node at (-1  ,  .8) {$Z_3$};
    \node at (-1.5,   0) [style=morphism, minimum width=1.6cm] {$m$};
    \node at ( 1.5,   0) [style=morphism, minimum width=1.6cm] {$m$};
    \node at ( 3.5,  .6) [style=morphism, minimum width=1.6cm] {$m$};
    \node at ( 5.5,  .6) [style=morphism, minimum width=1.6cm] {$m$};
  \end{pgfonlayer}
  \begin{pgfonlayer}{edgelayer}
    \draw (1.center) to (2.center);
    \draw (3.center) to (4.center);
    \draw (5.center) to (6.center);
    \draw (7.center) to (8.center);
    \draw (9.center) to (10.center);
    \draw [out=180, in=-90, looseness=.9] (10.center) to (12.center);
    \draw (11.center) to (14.center);
    \draw (12.center) to (15.center);
    \draw (13.center) to (16.center);
    \draw [out=  0, in=-90] (10.center) to (17.center);
    \draw (18.center) to (21.center);
    \draw (19.center) to (22.center);
    \draw (20.center) to (23.center);
    \draw (24.center) to (27.center);
    \draw (25.center) to (28.center);
    \draw (26.center) to (29.center);
    \draw [out=180, in=-90] (17.center) to (19.center);
    \draw [out=  0, in=-90] (17.center) to (25.center);
  \end{pgfonlayer}
\end{tikzpicture}
\end{adjustbox}
\end{equation}
\end{proof}

%% file: main.bbl
\begin{thebibliography}{45}
\providecommand{\natexlab}[1]{#1}
\providecommand{\url}[1]{\texttt{#1}}
\expandafter\ifx\csname urlstyle\endcsname\relax
  \providecommand{\doi}[1]{doi: #1}\else
  \providecommand{\doi}{doi: \begingroup \urlstyle{rm}\Url}\fi

\bibitem[Ad{\'a}mek et~al.(1990)Ad{\'a}mek, Herrlich, and
  Strecker]{herrlich1990abstract}
Ad{\'a}mek, J., Herrlich, H., and Strecker, G.~E.
\newblock \emph{Abstract and Concrete Categories: The Joy of Cats}.
\newblock John Wiley and Sons, 1990.
\newblock URL
  \url{http://www.tac.mta.ca/tac/reprints/articles/17/tr17abs.html}.

\bibitem[Awodey(2006)]{awodey2006category}
Awodey, S.
\newblock \emph{Category theory}.
\newblock Oxford University Press, 2006.
\newblock URL \url{https://doi.org/10.1093/acprof:oso/9780198568612.001.0001}.

\bibitem[Baez(2017)]{baez2017applied}
Baez, J.
\newblock Applied category theory 2018 | the n-category café, 2017.
\newblock URL
  \url{https://golem.ph.utexas.edu/category/2017/09/applied_category_theory_1.html}.

\bibitem[Baez et~al.(2011)Baez, Fritz, and Leinster]{baez2011characterization}
Baez, J.~C., Fritz, T., and Leinster, T.
\newblock A characterization of entropy in terms of information loss.
\newblock \emph{Entropy}, 13\penalty0 (11):\penalty0 1945--1957, 2011.
\newblock URL \url{https://doi.org/10.3390/e13111945}.
\newblock \url{https://arxiv.org/abs/1106.1791}.

\bibitem[Bengio et~al.(2013)Bengio, Courville, and
  Vincent]{bengio2013representation}
Bengio, Y., Courville, A., and Vincent, P.
\newblock Representation learning: A review and new perspectives.
\newblock \emph{IEEE transactions on pattern analysis and machine
  intelligence}, 35\penalty0 (8):\penalty0 1798--1828, 2013.
\newblock URL \url{https://doi.org/10.1109/TPAMI.2013.50}.
\newblock \url{https://arxiv.org/abs/1206.5538}.

\bibitem[Borceux(1994)]{borceux1994handbook}
Borceux, F.
\newblock \emph{Handbook of categorical algebra: volume 1, Basic category
  theory}, volume~1.
\newblock Cambridge University Press, 1994.
\newblock URL \url{https://doi.org/10.1017/CBO9780511525858}.

\bibitem[Bradley(2018)]{bradley2018applied}
Bradley, T.-D.
\newblock What is applied category theory?
\newblock \emph{arXiv preprint arXiv:1809.05923}, 2018.
\newblock URL \url{https://arxiv.org/abs/1809.05923}.

\bibitem[Carbonneau et~al.(2022)Carbonneau, Zaidi, Boilard, and
  Gagnon]{carbonneau2022measuring}
Carbonneau, M.-A., Zaidi, J., Boilard, J., and Gagnon, G.
\newblock Measuring disentanglement: A review of metrics.
\newblock \emph{IEEE Transactions on Neural Networks and Learning Systems},
  2022.
\newblock URL \url{https://doi.org/10.1109/TNNLS.2022.3218982}.
\newblock \url{https://arxiv.org/abs/2012.09276}.

\bibitem[Chen et~al.(2018)Chen, Li, Grosse, and Duvenaud]{chen2018isolating}
Chen, R.~T., Li, X., Grosse, R.~B., and Duvenaud, D.~K.
\newblock Isolating sources of disentanglement in variational autoencoders.
\newblock In \emph{Neural Information Processing Systems}, 2018.
\newblock URL
  \url{https://proceedings.neurips.cc/paper/2018/hash/1ee3dfcd8a0645a25a35977997223d22-Abstract.html}.

\bibitem[Cho \& Jacobs(2019)Cho and Jacobs]{cho2019disintegration}
Cho, K. and Jacobs, B.
\newblock Disintegration and bayesian inversion via string diagrams.
\newblock \emph{Mathematical Structures in Computer Science}, 29\penalty0
  (7):\penalty0 938--971, 2019.
\newblock URL \url{https://doi.org/10.1017/S0960129518000488}.
\newblock \url{https://arxiv.org/abs/1709.00322}.

\bibitem[Cohen \& Welling(2014)Cohen and Welling]{cohen2014learning}
Cohen, T. and Welling, M.
\newblock Learning the irreducible representations of commutative lie groups.
\newblock In \emph{International Conference on Machine Learning}, 2014.
\newblock URL \url{https://proceedings.mlr.press/v32/cohen14.html}.

\bibitem[Cohen \& Welling(2015)Cohen and Welling]{cohen2015transformation}
Cohen, T. and Welling, M.
\newblock Transformation properties of learned visual representations.
\newblock In \emph{International Conference on Learning Representations}, 2015.
\newblock URL \url{http://arxiv.org/abs/1412.7659}.

\bibitem[de~Haan et~al.(2020)de~Haan, Cohen, and Welling]{de2020natural}
de~Haan, P., Cohen, T., and Welling, M.
\newblock Natural graph networks.
\newblock \emph{Neural Information Processing Systems}, 33:\penalty0
  3636--3646, 2020.
\newblock URL
  \url{https://proceedings.neurips.cc/paper/2020/hash/2517756c5a9be6ac007fe9bb7fb92611-Abstract.html}.

\bibitem[Dittadi et~al.(2021)Dittadi, Tr{\"a}uble, Locatello, Wuthrich,
  Agrawal, Winther, Bauer, and Sch{\"o}lkopf]{dittadi2021transfer}
Dittadi, A., Tr{\"a}uble, F., Locatello, F., Wuthrich, M., Agrawal, V.,
  Winther, O., Bauer, S., and Sch{\"o}lkopf, B.
\newblock On the transfer of disentangled representations in realistic
  settings.
\newblock In \emph{International Conference on Learning Representations}, 2021.
\newblock URL \url{https://openreview.net/forum?id=8VXvj1QNRl1}.

\bibitem[Do \& Tran(2020)Do and Tran]{do2020theory}
Do, K. and Tran, T.
\newblock Theory and evaluation metrics for learning disentangled
  representations.
\newblock In \emph{International Conference on Learning Representations}, 2020.
\newblock URL \url{https://openreview.net/forum?id=HJgK0h4Ywr}.

\bibitem[Duan et~al.(2020)Duan, Matthey, Saraiva, Watters, Burgess, Lerchner,
  and Higgins]{duan2019unsupervised}
Duan, S., Matthey, L., Saraiva, A., Watters, N., Burgess, C., Lerchner, A., and
  Higgins, I.
\newblock Unsupervised model selection for variational disentangled
  representation learning.
\newblock In \emph{International Conference on Learning Representations}, 2020.
\newblock URL \url{https://openreview.net/forum?id=SyxL2TNtvr}.

\bibitem[Dudzik \& Veli{\v{c}}kovi{\'c}(2022)Dudzik and
  Veli{\v{c}}kovi{\'c}]{dudzik2022graph}
Dudzik, A.~J. and Veli{\v{c}}kovi{\'c}, P.
\newblock Graph neural networks are dynamic programmers.
\newblock In \emph{Neural Information Processing Systems}, 2022.
\newblock URL \url{https://openreview.net/forum?id=wu1Za9dY1GY}.

\bibitem[Eastwood \& Williams(2018)Eastwood and
  Williams]{eastwood2018framework}
Eastwood, C. and Williams, C.~K.
\newblock A framework for the quantitative evaluation of disentangled
  representations.
\newblock In \emph{International Conference on Learning Representations}, 2018.
\newblock URL \url{https://openreview.net/forum?id=By-7dz-AZ}.

\bibitem[Fong \& Spivak(2019)Fong and Spivak]{fong2019invitation}
Fong, B. and Spivak, D.~I.
\newblock \emph{An invitation to applied category theory: seven sketches in
  compositionality}.
\newblock Cambridge University Press, 2019.
\newblock URL \url{https://doi.org/10.1017/9781108668804}.
\newblock \url{https://arxiv.org/abs/1803.05316}.

\bibitem[Franz(2002)]{franz2002stochastic}
Franz, U.
\newblock What is stochastic independence?
\newblock In \emph{Non-commutativity, infinite-dimensionality and probability
  at the crossroads}, pp.\  254--274. World Scientific, 2002.
\newblock URL \url{https://doi.org/10.1142/9789812705242_0008}.
\newblock \url{https://arxiv.org/abs/math/0206017}.

\bibitem[Fritz(2020)]{fritz2020synthetic}
Fritz, T.
\newblock A synthetic approach to markov kernels, conditional independence and
  theorems on sufficient statistics.
\newblock \emph{Mathematics}, 370:\penalty0 107239, 2020.
\newblock URL \url{https://doi.org/10.1016/j.aim.2020.107239}.
\newblock \url{https://arxiv.org/abs/1908.07021}.

\bibitem[Gatys et~al.(2016)Gatys, Ecker, and Bethge]{gatys2016image}
Gatys, L.~A., Ecker, A.~S., and Bethge, M.
\newblock Image style transfer using convolutional neural networks.
\newblock In \emph{Computer Vision and Pattern Recognition}, 2016.
\newblock URL \url{https://doi.org/10.1109/CVPR.2016.265}.

\bibitem[Gavranovi{\'c}(2019)]{gavranovic2019compositional}
Gavranovi{\'c}, B.
\newblock Compositional deep learning.
\newblock Master's thesis, University of Zagreb, 2019.
\newblock URL \url{https://arxiv.org/abs/1907.08292}.

\bibitem[Giry(1982)]{giry1982categorical}
Giry, M.
\newblock A categorical approach to probability theory.
\newblock \emph{Categorical Aspects of Topology and Analysis}, pp.\  68--85,
  1982.
\newblock URL \url{https://doi.org/10.1007/BFb0092872}.

\bibitem[Higgins et~al.(2017)Higgins, Matthey, Pal, Burgess, Glorot, Botvinick,
  Mohamed, and Lerchner]{higgins2017betavae}
Higgins, I., Matthey, L., Pal, A., Burgess, C., Glorot, X., Botvinick, M.,
  Mohamed, S., and Lerchner, A.
\newblock beta-{VAE}: Learning basic visual concepts with a constrained
  variational framework.
\newblock In \emph{International Conference on Learning Representations}, 2017.
\newblock URL \url{https://openreview.net/forum?id=Sy2fzU9gl}.

\bibitem[Higgins et~al.(2018)Higgins, Amos, Pfau, Racaniere, Matthey, Rezende,
  and Lerchner]{higgins2018towards}
Higgins, I., Amos, D., Pfau, D., Racaniere, S., Matthey, L., Rezende, D., and
  Lerchner, A.
\newblock Towards a definition of disentangled representations.
\newblock \emph{arXiv preprint arXiv:1812.02230}, 2018.
\newblock URL \url{https://arxiv.org/abs/1812.02230}.

\bibitem[Kumar et~al.(2018)Kumar, Sattigeri, and
  Balakrishnan]{kumar2018variational}
Kumar, A., Sattigeri, P., and Balakrishnan, A.
\newblock Variational inference of disentangled latent concepts from unlabeled
  observations.
\newblock In \emph{International Conference on Learning Representations}, 2018.
\newblock URL \url{https://openreview.net/forum?id=H1kG7GZAW}.

\bibitem[Leinster(2014)]{leinster2014basic}
Leinster, T.
\newblock \emph{Basic category theory}, volume 143.
\newblock Cambridge University Press, 2014.
\newblock URL \url{https://doi.org/10.1017/CBO9781107360068}.
\newblock \url{https://arxiv.org/abs/1612.09375}.

\bibitem[Leinster(2016)]{leinster2016monoidal}
Leinster, T.
\newblock Monoidal categories with projections | the n-category café, 2016.
\newblock URL
  \url{https://golem.ph.utexas.edu/category/2016/08/monoidal_categories_with_proje.html}.

\bibitem[Li et~al.(2020)Li, Murkute, Gyawali, and Wang]{li2020progressive}
Li, Z., Murkute, J.~V., Gyawali, P.~K., and Wang, L.
\newblock Progressive learning and disentanglement of hierarchical
  representations.
\newblock In \emph{International Conference on Learning Representations}, 2020.
\newblock URL \url{https://openreview.net/forum?id=SJxpsxrYPS}.

\bibitem[Liu(2011)]{liu2011learning}
Liu, T.-Y.
\newblock \emph{Learning to Rank for Information Retrieval}.
\newblock Springer, 2011.
\newblock URL \url{https://doi.org/10.1007/978-3-642-14267-3}.

\bibitem[Liu et~al.(2018)Liu, Van De~Weijer, and Bagdanov]{liu2018leveraging}
Liu, X., Van De~Weijer, J., and Bagdanov, A.~D.
\newblock Leveraging unlabeled data for crowd counting by learning to rank.
\newblock In \emph{Computer Vision and Pattern Recognition}, 2018.
\newblock URL \url{https://doi.org/10.1109/CVPR.2018.00799}.

\bibitem[Locatello et~al.(2019)Locatello, Bauer, Lucic, Raetsch, Gelly,
  Sch{\"o}lkopf, and Bachem]{locatello2019challenging}
Locatello, F., Bauer, S., Lucic, M., Raetsch, G., Gelly, S., Sch{\"o}lkopf, B.,
  and Bachem, O.
\newblock Challenging common assumptions in the unsupervised learning of
  disentangled representations.
\newblock In \emph{International Conference on Machine Learning}, 2019.
\newblock URL \url{https://proceedings.mlr.press/v97/locatello19a.html}.

\bibitem[Montero et~al.(2020)Montero, Ludwig, Costa, Malhotra, and
  Bowers]{montero2020role}
Montero, M.~L., Ludwig, C.~J., Costa, R.~P., Malhotra, G., and Bowers, J.
\newblock The role of disentanglement in generalisation.
\newblock In \emph{International Conference on Learning Representations}, 2020.
\newblock URL \url{https://openreview.net/forum?id=qbH974jKUVy}.

\bibitem[Patterson(2017)]{patterson2017knowledge}
Patterson, E.
\newblock Knowledge representation in bicategories of relations.
\newblock \emph{arXiv preprint arXiv:1706.00526}, 2017.
\newblock URL \url{https://arxiv.org/abs/1706.00526}.

\bibitem[Piedeleu \& Zanasi(2023)Piedeleu and Zanasi]{piedeleu2023introduction}
Piedeleu, R. and Zanasi, F.
\newblock An introduction to string diagrams for computer scientists.
\newblock \emph{arXiv preprint arXiv:2305.08768}, 2023.
\newblock URL \url{https://arxiv.org/abs/2305.08768}.

\bibitem[Ridgeway \& Mozer(2018)Ridgeway and Mozer]{ridgeway2018learning}
Ridgeway, K. and Mozer, M.~C.
\newblock Learning deep disentangled embeddings with the f-statistic loss.
\newblock In \emph{Neural Information Processing Systems}, 2018.
\newblock URL
  \url{https://proceedings.neurips.cc/paper/2018/hash/2b24d495052a8ce66358eb576b8912c8-Abstract.html}.

\bibitem[Selinger(2010)]{selinger2010survey}
Selinger, P.
\newblock A survey of graphical languages for monoidal categories.
\newblock In \emph{New structures for physics}, pp.\  289--355. Springer, 2010.
\newblock URL \url{https://doi.org/10.1007/978-3-642-12821-9_4}.
\newblock \url{https://arxiv.org/abs/0908.3347}.

\bibitem[Sepliarskaia et~al.(2019)Sepliarskaia, Kiseleva, and
  de~Rijke]{sepliarskaia2019evaluating}
Sepliarskaia, A., Kiseleva, J., and de~Rijke, M.
\newblock Evaluating disentangled representations.
\newblock \emph{arXiv preprint arXiv:1910.05587}, 2019.
\newblock URL \url{https://arxiv.org/abs/1910.05587}.

\bibitem[Shiebler et~al.(2021)Shiebler, Gavranovi{\'c}, and
  Wilson]{shiebler2021category}
Shiebler, D., Gavranovi{\'c}, B., and Wilson, P.
\newblock Category theory in machine learning.
\newblock \emph{arXiv preprint arXiv:2106.07032}, 2021.
\newblock URL \url{https://arxiv.org/abs/2106.07032}.

\bibitem[Shu et~al.(2020)Shu, Chen, Kumar, Ermon, and Poole]{shu2020weakly}
Shu, R., Chen, Y., Kumar, A., Ermon, S., and Poole, B.
\newblock Weakly supervised disentanglement with guarantees.
\newblock In \emph{International Conference on Learning Representations}, 2020.
\newblock URL \url{https://openreview.net/forum?id=HJgSwyBKvr}.

\bibitem[Suter et~al.(2019)Suter, Miladinovic, Sch{\"o}lkopf, and
  Bauer]{suter2019robustly}
Suter, R., Miladinovic, D., Sch{\"o}lkopf, B., and Bauer, S.
\newblock Robustly disentangled causal mechanisms: Validating deep
  representations for interventional robustness.
\newblock In \emph{International Conference on Machine Learning}, 2019.
\newblock URL \url{http://proceedings.mlr.press/v97/suter19a.html}.

\bibitem[Tokui \& Sato(2022)Tokui and Sato]{tokui2022disentanglement}
Tokui, S. and Sato, I.
\newblock Disentanglement analysis with partial information decomposition.
\newblock In \emph{International Conference on Learning Representations}, 2022.
\newblock URL \url{https://openreview.net/forum?id=pETy-HVvGtt}.

\bibitem[Tr{\"a}uble et~al.(2021)Tr{\"a}uble, Creager, Kilbertus, Locatello,
  Dittadi, Goyal, Sch{\"o}lkopf, and Bauer]{trauble2021disentangled}
Tr{\"a}uble, F., Creager, E., Kilbertus, N., Locatello, F., Dittadi, A., Goyal,
  A., Sch{\"o}lkopf, B., and Bauer, S.
\newblock On disentangled representations learned from correlated data.
\newblock In \emph{International Conference on Machine Learning}, 2021.
\newblock URL \url{https://proceedings.mlr.press/v139/trauble21a.html}.

\bibitem[Wang et~al.(2021)Wang, Yue, Huang, Sun, and Zhang]{wang2021self}
Wang, T., Yue, Z., Huang, J., Sun, Q., and Zhang, H.
\newblock Self-supervised learning disentangled group representation as
  feature.
\newblock In \emph{Neural Information Processing Systems}, 2021.
\newblock URL \url{https://openreview.net/forum?id=RQfcckT1M_4}.

\end{thebibliography}
